\documentclass{article}

\usepackage[main, final]{neurips_2026}

\usepackage[utf8]{inputenc} 
\usepackage[T1]{fontenc}    
\usepackage{hyperref}       
\usepackage{url}            
\usepackage{booktabs}       
\usepackage{amsfonts}       
\usepackage{nicefrac}       
\usepackage{microtype}      
\usepackage{xcolor}         
\usepackage{graphicx}
\usepackage{subcaption}
\usepackage{algpseudocode}

\usepackage{array}
\usepackage{amsmath}
\usepackage{amssymb}
\usepackage{mathtools}
\usepackage{amsthm}
\usepackage{subfiles}
\usepackage{tabularx}
\usepackage{multirow}
\usepackage{bbm}
\usepackage{wrapfig}
\usepackage[textsize=tiny]{todonotes}
\usepackage[capitalize,noabbrev]{cleveref}
\usepackage{thmtools}
\usepackage{thm-restate}

\usepackage{algorithm}
\theoremstyle{plain}
\newtheorem{theorem}{Theorem}[section]
\newtheorem{proposition}[theorem]{Proposition}
\newtheorem{lemma}[theorem]{Lemma}

\theoremstyle{definition}
\newtheorem{definition}[theorem]{Definition}

\newtheorem{problem}[theorem]{Problem}

\theoremstyle{remark}
\newtheorem{remark}[theorem]{Remark}
\newcommand{\model}{SAGE}
\newif\ifshowcomment
\showcommenttrue

\newcommand{\meanstd}[2]{#1{\tiny$\pm$#2}}

\title{\model: Mitigating Long-Horizon Reasoning Biases via Topological Guidance}

\author{Xinyue Zeng \\CS Department\\Virginia Tech\\
\And
Jiawei Zhang \\
  CS Department\\
  University of Wisconsin Madison\\
  \And
  Yujun Yan \\
  CS Department\\
  Dartmouth College\\
  \And
  Dawei Zhou \\
  CS Department\\
  Virginia Tech\\
}

\begin{document}

\maketitle

\begin{abstract}
Long-horizon reasoning remains a central challenge for large language models (LLMs) under sparse-reward regimes.
We argue that this brittleness arises from two biases induced by complex reasoning spaces: 
an exploration bias, where models are drawn toward locally plausible but structurally unstable branches, and a compounding bias, where small local deviations accumulate across depth and suppress rare rewards.
We introduce Symbolic Closure Analysis (SCA) as a theoretical lens characterizing how branching structures and sparse rewards induce these biases in long-horizon reasoning with local admissibility, and as a design principle for structural priors in less formal reasoning tasks. Motivated by this analysis, we propose \model{} (\textbf{S}tructural \textbf{A}dmissibility-\textbf{G}uided \textbf{E}xploration), a unified framework that injects structural guidance to alleviate exploration bias and compounding bias in long-horizon reasoning. \model{} combines two complementary structural guidance: algebraic sparsification, which projects locally admissible candidates onto operator-indexed algebraic subspaces to suppress spurious branching and mitigate exploration bias, and hyperbolic structural guidance, which embeds reasoning states into a negatively curved space to provide dense depth-wise signals and mitigate compounding bias. Across 12 benchmarks and 10 model families, \model{} outperforms competitive baselines. In particular, \model{} achieves up to an 8-fold improvement on the Andrews-Curtis problem, an open real-world long-horizon task. Code is available at: \url{https://github.com/Susan571/SAGE-NeurIPS2026}. 
\end{abstract}

\section{Introduction}
\label{sec:intro}
Long-horizon reasoning remains a central challenge for large language models (LLMs), where success often requires discovering deep, structured reasoning trajectories across many steps.
Post-training learning has become a dominant paradigm for improving LLM reasoning, with strong progress in mathematics and coding through outcome-based objectives~\citep{lyu2025exploring, zhang2025deep, zeng2025lensllm, lin2026plan}.
However, as reasoning horizons grow, outcome-based post-training becomes brittle in sparse-reward regimes: limited intermediate rewards provide insufficient guidance for preserving long-term feasibility, leading to locally plausible but globally unstable patterns~\citep{suo2025long,zeng2026halluguard}.

This brittleness reflects two distinct but coupled long-horizon reasoning biases, illustrated in Figure~\ref{fig:long_horizon_biases} with the Andrews-Curtis (AC) trivialization task as a running example, where the task is to reach a trivial presentation through a long sequence of admissible transformations. The first is \emph{exploration bias}, which arises from the structural complexity of the reasoning space: as the reasoning tree expands with depth, successful trajectories occupy a narrow feasible region, while much larger regions contain locally plausible but globally unproductive branches~\citep{dziri2023faith}, so outcome-based post-training tends to favor trajectories that are easy to sample rather than trajectories that remain extendable to success. The second is \emph{compounding bias}, which arises from sparse-reward regimes with limited intermediate rewards: small local deviations accumulate across depth and progressively move the trajectory away from long-term feasibility~\citep{casper2023open, weaver2013optimal}. Together, these biases explain why LLMs may produce plausible intermediate steps while still failing at end-to-end long-horizon reasoning, and existing mitigation efforts address them along two corresponding axes. To counter exploration bias, one line of work introduces dense supervision that decomposes long-horizon reasoning into locally verifiable steps, ranging from interactive verifiers that enforce rigorous logical transitions~\citep{yang2023lean, hsiang2025leandojo} to learned verifiers or LLM-as-a-Judge that guide Monte Carlo Tree Search via step-wise critiques~\citep{lightman2023lets}; however, reliable process supervision remains structurally expensive to scale, ultimately shifting the bottleneck without resolving the fundamental difficulty of learning from sparse rewards. To counter compounding bias, a second line of work adapts outcome-based RL to sparse-reward regimes by augmenting the objective with auxiliary heuristics, including iterative bootstrapping methods like STaR or ReST~\citep{zelikman2022star, gulcehre2023reinf} and intrinsic motivation mechanisms based on entropy regularization or syntactic constraints~\citep{she2025rprm, zhang2025count, yue2025does, gai2025differential}; yet these methods do not model the structure of the long-horizon reasoning space, leaving the geometry of feasible trajectories implicit.

This gap motivates two fundamental research questions:
\textbf{Q1:} \textit{Can we characterize the dominant factors that drive long-horizon reasoning failures under structural complexity and sparse-reward regimes?}
\textbf{Q2:} \textit{Can this characterization be instantiated as a unified framework that injects structural guidance to alleviate exploration bias and compounding bias?}

\begin{wrapfigure}{r}{0.4\linewidth}
    \vspace{-18pt}
    \centering        \includegraphics[width=\linewidth]{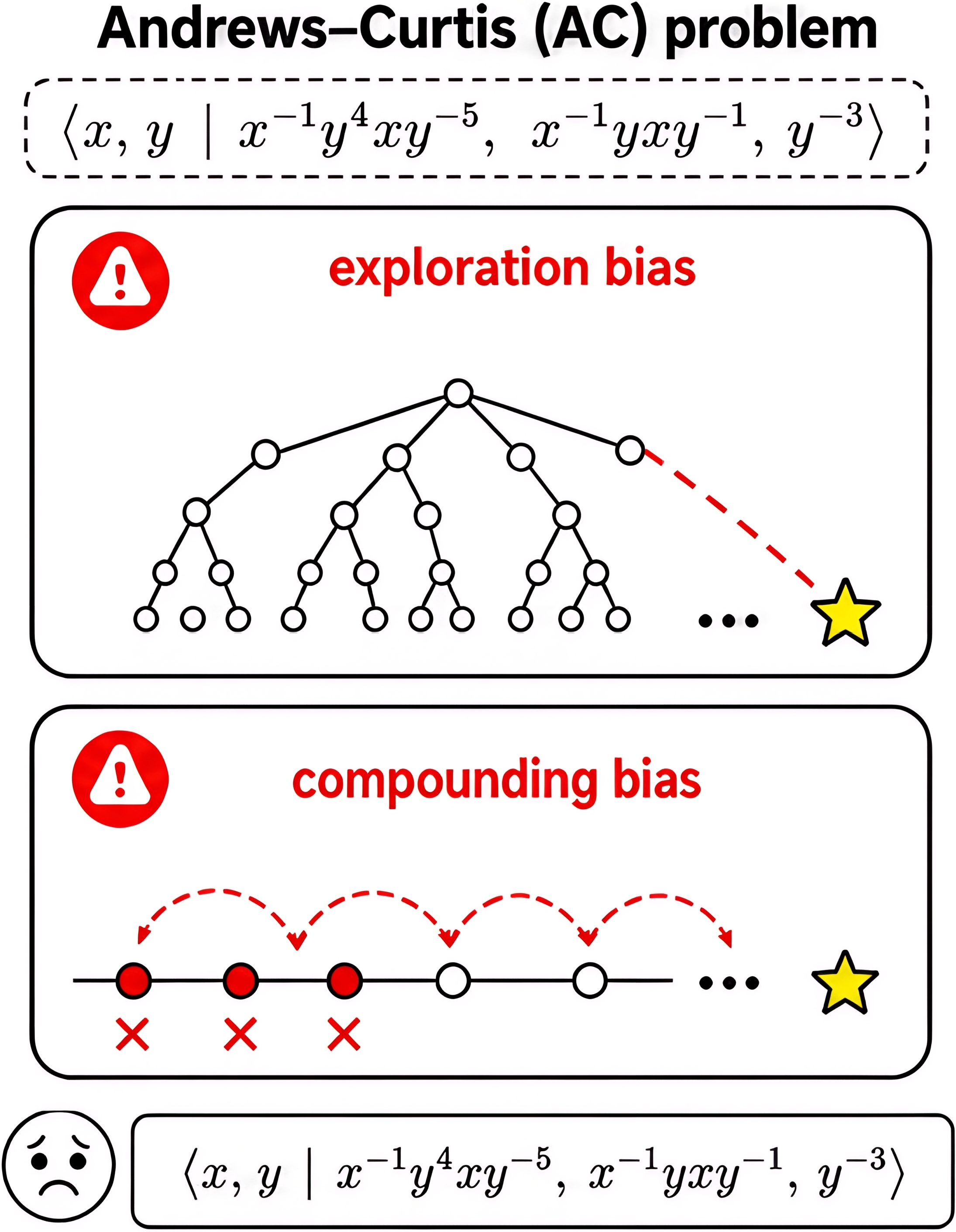}
    \caption{AC problem illustrates two long-horizon reasoning biases: exploration bias from many locally valid but low-promise branches, and compounding bias from early plausible deviations whose failures appear in later steps.}   \label{fig:long_horizon_biases}
    \vspace{-15pt}
\end{wrapfigure}


To address this gap, we first introduce \emph{Symbolic Closure Analysis} (SCA), a theoretical framework for characterizing the dominant factors behind long-horizon reasoning failures under structural complexity and sparse-reward regimes. SCA models reasoning as a sequence of locally admissible transformations and studies how feasible support evolves in expanding reasoning spaces. Our analysis shows that structural complexity dilutes feasible support across high-volume but unproductive branches, producing exploration bias, while sparse-reward regimes provide limited intermediate rewards, allowing local deviations to accumulate with depth and produce compounding bias.

Motivated by SCA, we propose \textit{\textbf{S}tructural \textbf{A}dmissibility-\textbf{G}uided \textbf{E}xploration} (\model), a unified framework for alleviating long-horizon reasoning biases. \model{} injects structural guidance during policy optimization so that the learned policy internalizes useful properties of structured reasoning space, through two complementary structural guidance: \textit{algebraic sparsification}, which reduces spurious branching and alleviates exploration bias, and \textit{hyperbolic structural guidance}, which provides dense depth-aware signals and alleviates compounding bias. Together, these components translate the SCA diagnosis into trainable guidance signals.

We evaluate \model{} across 10 model families and 12 benchmarks spanning closed-form mathematical reasoning, free-form natural reasoning, and open long-horizon reasoning task. \model{} always outperforms competitive baselines and particularly achieves up to 8-fold improvement on AC problem, an open real-world long-horizon reasoning task. 

Our contributions are threefold:
\begin{itemize}
    \vspace{-0.5em}
    \item We introduce SCA as a theoretical lens to characterize the dominant factors behind long-horizon reasoning failures under structural complexity and sparse-reward regimes.
    \vspace{-0.37em}
    \item We propose \model, a unified framework for alleviating long-horizon reasoning biases.
    \vspace{-0.37em}
    \item Evaluation across 12 benchmarks and 10 models show that \model{} consistently outperforms competitive baseline. Code is open-sourced at: \url{https://github.com/Susan571/SAGE-NeurIPS2026}.
\end{itemize}

\section{Preliminaries}
\label{sec:pre}
\subsection{Long-Horizon Reasoning}
\label{sec:lhr_prelim}
Following prior work~\citep{yao2023tree, lightman2023lets}, we study long-horizon reasoning as a sequential decision process over discrete symbolic manipulations. Let $\mathfrak S$ denote the task-specific symbolic interface with local admissibility predicate $\mathrm{Adm}_{\mathfrak S}$. At each step $t \in \{0,\ldots,T-1\}$, the system selects $a_t \in \mathcal A(s_t)$ with $\mathrm{Adm}_{\mathfrak S}(s_t,a_t)=1$ and transitions via $s_{t+1} = \Phi(s_t,a_t)$, yielding a trajectory $\tau=(s_0,a_0,\ldots,s_T)$ of maximum path length $T$. For analytical convenience, we cast this as a finite-horizon MDP $\mathcal M=\langle \mathcal S,\mathcal A,\mathcal P,\mathcal R,T\rangle$ over reasoning contexts, symbolic manipulations, transition dynamics, and task-level outcome signal.
The difficulty grows rapidly with $T$, shaped by two salient properties. The first is \emph{structural complexity}: the reachable search space scales as $|\Omega| = O(\bar{A}^T)$ for average branching factor $\bar{A}$, while successful trajectories form only a small subset $\mathcal{T}^*$ (treated as an analytical object, not as supervision), occupying a vanishing fraction of the reachable manifold as $T$ grows. The second is \emph{sparse outcome feedback}: meaningful supervision is often available only at the trajectory end. Under terminal reward $r(\tau) = \mathbf{1}[\tau \in \mathcal{T}^*]$, informative feedback is inherently rare because successful trajectories themselves are rare, making credit assignment progressively harder and often yielding ineffective optimization in sparse-reward regimes~\citep{uesato2022solving, weaver2013optimal}.

\subsection{Biases in Long-Horizon Reasoning}
\label{sec:optimization_biases}


Structural complexity and sparse outcome feedback do not merely make long-horizon reasoning more difficult; they induce recurring distortions in how trajectories are explored and preserved~\citep{zhou2025dissecting,brantley2025accelerating,liu2025evaluating,jahin2025evaluating}. Figure~\ref{fig:long_horizon_biases} illustrates these distortions in the AC problem, where an LLM policy transforms a group presentation toward the trivial presentation through legal symbolic moves such as inverting a relator, multiplying relators, or conjugating by a generator. Although many moves are locally valid, only a small subset continues to simplify the presentation over long horizons, giving rise to two coupled biases: \emph{structural complexity} mainly induces an \emph{exploration bias}, while \emph{sparse outcome feedback} mainly induces a \emph{compounding bias}, and the two interact as depth increases.


\textbf{Exploration bias.} As search volume grows exponentially, locally admissible trajectories that do not extend to success can overwhelmingly dominate the reachable set~\citep{dziri2023faith}: many moves are locally admissible, but only a few are structurally extendable. In the AC example, many legal moves branch from the same presentation, yet only a few continue to simplify the residual algebraic structure, so a policy sampling broadly from admissible moves spends most of its budget on locally plausible but structurally unstable paths.

\textbf{Compounding bias.} Sparse outcome feedback yields a different but equally persistent failure mode: with supervision only at the end of a long chain, small local deviations cannot be corrected early and instead accumulate. As shown in Figure~\ref{fig:long_horizon_biases}, an early AC move such as replacing one relator by its product with another may look locally plausible, yet redirect the trajectory into a region where subsequent legal moves preserve or amplify complexity, with failure surfacing only several steps later. This is especially problematic in KL-regularized post-training: when terminal rewards are weak, the update is dominated by the KL term~\citep{lyu2025exploring,schulman2017proximal}, and in the limit where the expected reward contribution vanishes, it degenerates toward preserving the reference policy,
$
\nabla J \approx -\beta \nabla D_{\mathrm{KL}}(\pi_\theta \| \pi_{\mathrm{ref}}),
$
allowing locally plausible deviations to persist rather than being corrected by task-level structure~\citep{casper2023open}.

Together, we formalize the problem as follows:

\begin{problem}[Alleviating Long-Horizon Reasoning Biases]
\label{prob:main}
\mbox{}\

\noindent
\textbf{Given:} (i) a task space $\mathcal{Q}$ with complex reasoning structures whose search volume scales exponentially with path length $T$ ($|\Omega| = O(\bar{A}^T)$); and (ii) a sparse-reward regime with a pretrained policy $\pi_0$ and sparse terminal reward $\mathcal{O}(\tau)$, where the initial success probability is negligible ($\mathbb{E}_{\tau \sim \pi_0}[\mathcal{O}(\tau)] \approx 0$).
\mbox{}\

\textbf{Find:} a reasoning policy $\pi^*$ that increases the probability of successful trajectories by alleviating both biases, without access to dense process labels or ground-truth solution paths:
$
\pi^* = \operatorname*{argmax}_{\pi} \mathbb{E}_{q \sim \mathcal{Q}} \left[ \mathbb{P}(\tau \in \mathcal{T}^* \mid \pi, q) \right].
$
\end{problem}

\section{Theoretical Analysis}
\label{sec:method}

We introduce \textit{Symbolic Closure Analysis} (SCA) as a theoretical lens for understanding the two biases identified in \Cref{sec:optimization_biases}. SCA characterizes the reasoning manifold through a prefix-closed feasible region induced by local admissibility, providing a principled lens for diagnosing why standard outcome-based learning is structurally biased under sparse-reward regimes.

\subsection{SCA: Symbolic Closure Analysis}
\label{sec:sca_closure}

To analyze long-horizon reasoning in sparse-reward regimes, we first revisit reference-regularized post-training~\citep{ziegler2019fine,ouyang2022training}, which stabilizes learning through a likelihood-based closure $\Omega_{\text{stat}} = { \tau : \mathbb{P}{\pi{\text{ref}}}(\tau) \ge \epsilon }$ but preserves trajectories that are easy to sample rather than those feasible over long horizons. We propose SCA, which instead defines feasibility through local admissibility. The resulting feasible region $\mathcal{F}$ is prefix-closed by construction, providing an analytical object against which exploration and compounding biases can be characterized.



\begin{definition}[SCA Feasible Closure]
\label{def:sca_feasible_closure}
Let $\mathfrak S$ denote a domain-specific symbolic system specifying admissible operators and a local admissibility predicate $\mathrm{Adm}_{\mathfrak S}$. For a trajectory $\tau=(s_0,a_0,\ldots,a_{T-1},s_T)$ with $a_t\in\mathcal A(s_t)$ and $s_{t+1}=\Phi(s_t,a_t)$, the SCA-feasible set $\mathcal F$ is
$
\tau\in\mathcal F
\Longleftrightarrow
\{ \forall t\in\{0,\ldots,T-1\},\ \mathrm{Adm}_{\mathfrak S}(s_t,a_t,s_{t+1})=1 \}.
$
\end{definition}

\begin{remark}[Local Admissibility]
$\mathrm{Adm}_{\mathfrak S}$ checks only whether a transition is locally well-formed under $\mathfrak S$; it does not certify global correctness, provide ground-truth steps, reveal successful trajectories, or use outcome labels.
\end{remark}

\begin{remark}[Prefix Closure]
$\mathrm{Adm}_{\mathcal{S}}$ induces a prefix-closed feasible region: if any prefix of $\tau$ violates local admissibility, all its continuations are infeasible. Equivalently, for any $\tau\in\mathcal{F}$, every prefix $\tau_{\le t}\in\mathcal{F}$.
\end{remark}

SCA admits three levels of instantiation. In explicit symbolic systems, $\mathrm{Adm}_{\mathfrak S}$ is given by the task interface and $\mathcal F$ is exact. In semi-structured domains such as closed-form mathematics, unresolved variables, equations, and answer-schema constraints serve as computable proxies for residual structure. In free-form natural reasoning, residuals and target anchors are estimated from prompt-conditioned semantic structure. 

\subsection{Theoretical Analysis of Biases under SCA}
\label{sec:theory}

We now show how SCA explains the two biases of \Cref{sec:optimization_biases}. Let $\mathcal{F}$ be the prefix-closed feasible set (Definition~\ref{def:sca_feasible_closure}) and $\mathcal{G}=\Omega\setminus\mathcal{F}$ its inadmissible complement, with $\mathcal{T}^*\subseteq\mathcal{F}$ unobserved. We treat $\mathcal{F}$ as a tractable structural proxy: induced by local admissibility in symbolic domains, and approximating the long-horizon extendable subset elsewhere.

For a rollout policy $\pi$, let $\mu_\pi := \mathbb{P}_{\tau\sim\pi}(\tau\in\mathcal{G})$ and $p_\pi := \mathbb{P}_{\tau\sim\pi}(\tau\in\mathcal{T}^*)$. Under SCA, exploration bias appears as $\mu_{\pi_0}\approx 1$—rollouts concentrate outside $\mathcal{F}$—while compounding bias appears as the persistence of early inadmissible deviations: terminal-only rewards let locally unstable prefixes survive long enough to dominate the trajectory distribution.

\textbf{Characterizing Exploration Bias via Feasible-Region Geometry.}
Structural complexity induces exploration bias because the reachable space grows exponentially with $T$, while admissible and successful trajectories occupy only a small fraction. Let $\tau_{\le t}$ denote a prefix and define the local feasibility indicator $\phi(\tau_{\le t})=\mathbf{1}[\tau_{\le t}\ \text{is locally admissible under }\mathrm{Adm}_{\mathcal S}]$. By prefix closure, $\phi(\tau_{\le t})=0 \Rightarrow \phi(\tau_{\le t'})=0$ for all $t'\ge t$, hence $\mathcal{G} = \{\tau\in\Omega:\exists t,\ \phi(\tau_{\le t})=0\}$.

Let $B_t$ and $B_t^{\mathcal{F}}$ denote the effective and locally admissible branching factors at depth $t$. When $B_t^{\mathcal{F}}\ll B_t$ for many $t$, a crude volume comparison gives
\begin{equation}
\frac{|\{\tau_{\le T}:\tau\in\mathcal{F}\}|}
{|\{\tau_{\le T}:\tau\in\Omega\}|}
\;\lesssim\;
\prod_{t=1}^{T}
\frac{B_t^{\mathcal{F}}}{B_t}.
\label{eq:volume_ratio}
\end{equation}
Thus, unless the base policy places exponentially increasing preference on $\mathcal{F}$, unconstrained rollouts concentrate outside the feasible region ($\mu_{\pi_0}\approx 1$). SCA identifies this volume mismatch as the structural origin of exploration bias and isolates feasible-support concentration as the property any successful mitigation must achieve and, as the next proposition shows, the structural quantity controlling gradient variance.

\begin{proposition}[Variance Decomposition over Feasible and Inadmissible Regions]
\label{prop:variance_decomposition}
Let $\pi$ be any rollout policy and $\hat{\rho}_\pi$ its empirical rollout distribution. For any score-function gradient term $g(\tau)$,
{\small
\begin{flalign}
\mathrm{Var}_{\tau\sim \hat{\rho}_{\pi}}\!\big[g(\tau)\big]
=&\ 
\mathbb{P}(\tau\in\mathcal{F})\,
\mathrm{Var}\!\big[g(\tau)\mid \tau\in\mathcal{F}\big]
\notag\\
&+
\mathbb{P}(\tau\in\mathcal{G})\,
\mathrm{Var}\!\big[g(\tau)\mid \tau\in\mathcal{G}\big]
\notag\\
&+
\mathbb{P}(\tau\in\mathcal{F})\mathbb{P}(\tau\in\mathcal{G})
\Big(
\mathbb{E}\!\big[g(\tau)\mid \tau\in\mathcal{F}\big]
-
\mathbb{E}\!\big[g(\tau)\mid \tau\in\mathcal{G}\big]
\Big)^2 .
\end{flalign}
}
Consequently, any policy with $\mathrm{supp}(\pi)\subseteq\mathcal{F}$ removes both the $\mathcal{G}$-conditioned variance and the between-region term.
\end{proposition}


\textbf{Characterizing Compounding Bias under Sparse-reward Regimes.}
Under sparse rewards, local deviations go uncorrected until terminal evaluation, allowing trajectories to drift from feasible structure without intermediate signal. KL-regularized post-training sharpens this: when successful trajectories are rare under the reference policy, the reward signal is too weak to pull the learned policy away from reference-model behavior.

Consider the KL-regularized objective
$
\max_{\pi}\ \mathbb{E}_{\tau\sim\pi}[R(\tau)] - \lambda D_{\mathrm{KL}}(\pi\|\pi_{\mathrm{ref}}),
\label{eq:kl_objective_sca}
$
with terminal-only $R(\tau)\in[0,R_{\max}]$ and successful set $S=\{\tau:R(\tau)>0\}$. Then $\mathbb{E}_{\tau\sim\pi}[R(\tau)] \le R_{\max}\pi(S)$, while the KL penalty is dense over the full rollout space. The following theorem makes this mismatch precise: when $S$ is rare under $\pi_{\mathrm{ref}}$, the full-support KL-regularized optimizer cannot move meaningfully away from $\pi_{\mathrm{ref}}$, so locally inadmissible prefixes under $\pi_{\rm ref}$ remain after optimization.

\begin{restatable}[Reference anchoring under rare terminal rewards]{theorem}{thmScaKlCollapse}
\label{thm:sca_kl_collapse}
For \Cref{eq:kl_objective_sca} with terminal-only $R(\tau)\in[0,R_{\max}]$, let $S=\{\tau:R(\tau)>0\}$ and $p=\mathbb{P}_{\tau\sim\pi_{\mathrm{ref}}}(S)$. The full-support optimizer
$
\pi^*_{\Omega}(\tau) = \pi_{\mathrm{ref}}(\tau)\exp(R(\tau)/\lambda) / \mathbb{E}_{\tau'\sim\pi_{\mathrm{ref}}}[\exp(R(\tau')/\lambda)]
$
satisfies
\begin{equation}
D_{\mathrm{TV}}(\pi^*_{\Omega},\pi_{\mathrm{ref}}) \le \bigl(e^{R_{\max}/\lambda}-1\bigr)p.
\label{eq:tv_collapse_bound}
\end{equation}
Hence if $(e^{R_{\max}/\lambda}-1)p \ll 1$, the optimizer remains close to $\pi_{\mathrm{ref}}$. In the high-KL or weak-reward regime $R_{\max}\ll\lambda$, this becomes $D_{\mathrm{TV}}(\pi^*_{\Omega},\pi_{\mathrm{ref}}) \le (R_{\max}/\lambda)p + O(R_{\max}^2 p/\lambda^2)$. Thus, when successful trajectories are rare under the reference, terminal-only KL-regularized optimization has limited leverage to move probability mass away from reference-likely prefixes.
\end{restatable}


\section{\model: Structural Admissibility-Guided Exploration}
\label{sec:alg}
\subsection{From SCA to \model}
\label{sec:sca_to_\model{}}

SCA turns the two biases into computational requirements: exploration bias requires feasible-support concentration without observing $\mathcal{T}^*$, and compounding bias requires prefix-level correction without dense process labels. Motivated by~\Cref{thm:sca_kl_collapse}, we address both through \textit{\textbf{S}tructural \textbf{A}dmissibility-\textbf{G}uided \textbf{E}xploration} (\model{}), which injects two complementary structural potentials into post-training.

\begin{wrapfigure}{r}{0.6\linewidth}
\vspace{-13pt}
    \centering        \includegraphics[width=\linewidth]{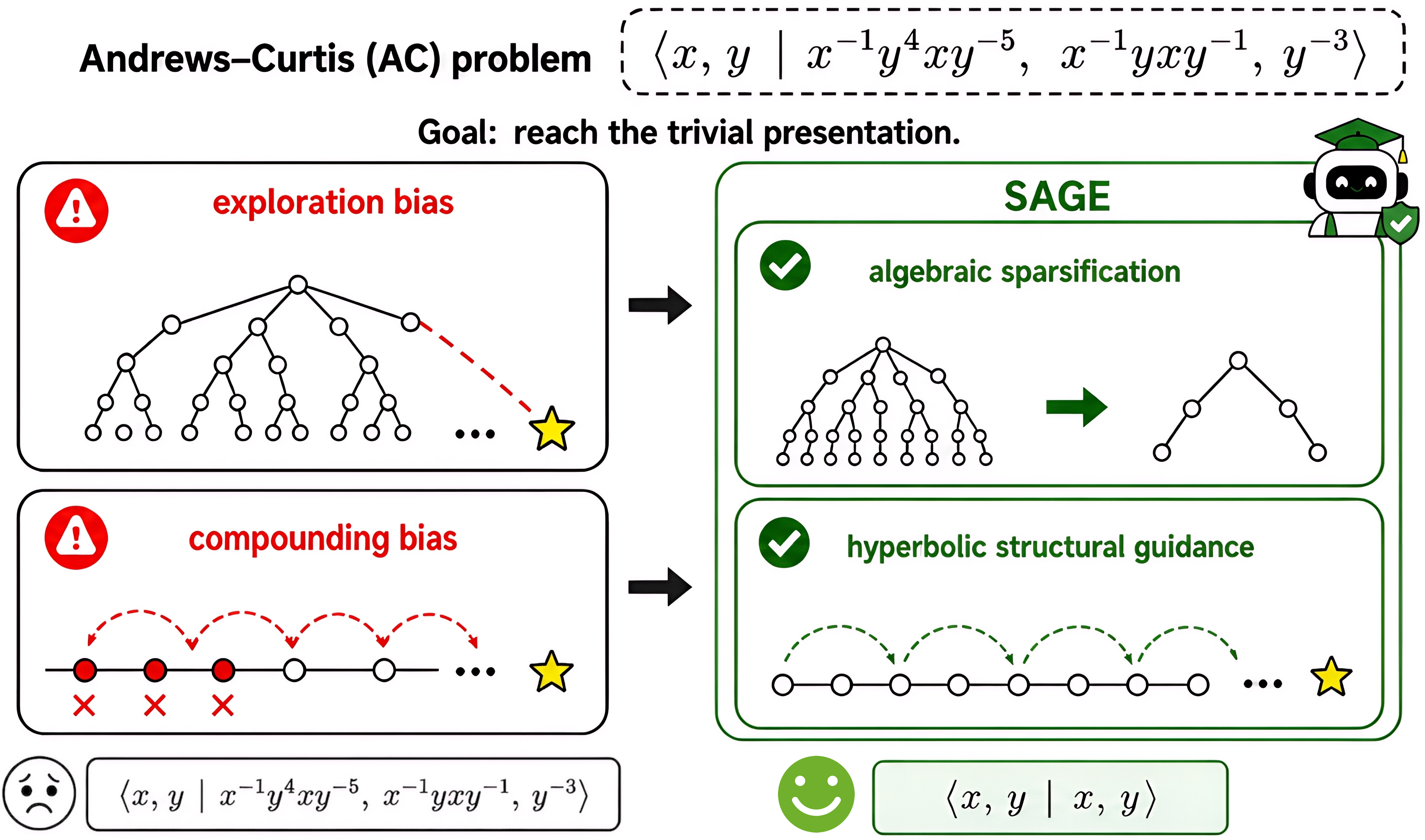}
    \caption{Overview of \model{} workflow. Outcome-only training induces exploration and compounding biases; \model{} injects algebraic sparsification and hyperbolic structural guidance during RL, yielding focused search, stable reasoning, and no additional inference-time cost.}
 \vspace{-10pt}   
    \label{fig:sage_overview}
\end{wrapfigure}

\textbf{Algebraic sparsification} $\Psi_{\mathcal{P}}$ targets exploration bias by biasing sampling toward operators that explain the current symbolic residual $r_t\in\mathbb{R}^d$ (\Cref{eq:volume_ratio}). For a candidate operator $L_j$ with associated subspace $S_j$:
\begin{equation}
\Psi_{\mathcal{P}}(r_t,S_j) = \frac{\|P_{S_j}r_t\|_2^2}{\|r_t\|_2^2+\varepsilon},
\label{eq:psi_p}
\end{equation}
where $P_{S_j}$ projects onto $S_j$. Larger $\Psi_{\mathcal{P}}$ indicates the operator addresses more of the unresolved residual, providing a soft compatibility score (not a replacement for $\mathrm{Adm}_{\mathfrak S}$) that promotes feasible-support concentration without inference-time pruning.

\textbf{Hyperbolic structural guidance} $\Psi_{\mathcal{H}}$ targets compounding bias by supplying prefix-level feedback before terminal rewards arrive (\Cref{thm:sca_kl_collapse}), exploiting hyperbolic geometry's natural fit for hierarchical reasoning structure. For state $s_t$, operator $L_j$, target structure $g$, and Poincaré embedding $\mathcal{E}(\cdot)$:
$
\Psi_{\mathcal{H}}(s_t,L_j,g) = \exp\!\left( -d_{\mathbb{D}}\bigl(\mathcal{E}(s_t\circ L_j),\mathcal{E}(g)\bigr) / \kappa \right),
$
with Poincaré distance $d_{\mathbb{D}}$ and sharpness $\kappa>0$. For Andrews--Curtis, $s_t$, $r_t$, $g$ are the current presentation, unresolved structure, and trivial presentation; for closed-form and free-form reasoning, $r_t$ and $g$ are training-time structural priors (Appendix~\ref{app:task_instantiation}). Both potentials are computable without $\mathcal{T}^*$ or process labels, and the resulting structural preferences are absorbed into the policy during post-training, so inference incurs no additional search or filtering.

\begin{proposition}[Non-vanishing structural advantage under sparse rewards]
\label{prop:vanish}
If $R_{\rm term}(\tau)=0$ across a rollout group and $\Psi_{\model}$ has nonzero within-group variance, the SAGE advantage $A_i = \bigl(\eta \bar{\Psi}_{\rm \model}(\tau_i) - \eta \tfrac{1}{G}\sum_j \bar{\Psi}_{\rm \model}(\tau_j)\bigr) / \bigl(\operatorname{std}_j(\eta \bar{\Psi}_{\rm \model}(\tau_j))+\epsilon\bigr)$ remains nonzero, providing an update signal even under uninformative terminal rewards.
\end{proposition}

\subsection{\model{}}
\label{sec:\model{}_training}

The two potentials combine into the \model{} mechanism: $\Psi_{\model}(s_t,a_t) = \alpha\,\Psi_{\mathcal{P}}(s_t,a_t) + \gamma\,\Psi_{\mathcal{H}}(s_t,a_t)$, with $\alpha,\gamma\ge 0$ controlling relative strength.

\textbf{Structure-Guided Sampling.}
\model{} modulates the rollout distribution via
$
P_{\mathrm{sample}}(a_t^{(k)}\mid s_t,\mathcal C_t) \propto \exp\!\left( \bar{\ell}_{\theta_{\mathrm{old}}}(a_t^{(k)}\mid s_t) + \lambda \Psi_{\model{}}(s_t,a_t^{(k)}) \right), \quad a_t^{(k)}\in\mathcal C_t,
$
where $\lambda\ge 0$ controls guidance strength and $\mathcal C_t$ is, for non-symbolic tasks, a finite candidate set sampled from $\pi_{\theta_{\mathrm{old}}}$ (normalization is over $\mathcal C_t$, not the full vocabulary). This implements both SCA requirements as soft biases-feasible-support concentration via $\Psi_{\mathcal{P}}$, prefix-level signal via $\Psi_{\mathcal{H}}$-rather than projecting onto $\mathcal{F}$ through hard constraints.

\textbf{Guidance-Augmented Advantage.}
For each $\tau_i$, \model{} forms a reward combining terminal outcome with structural guidance:
$
\widetilde{R}(\tau_i) = R_{\mathrm{term}}(\tau_i) + \eta\cdot \tfrac{1}{T_i} \sum_{t=1}^{T_i} \Psi_{\model}(s_t^i,a_t^i),
$
with group-relative advantage $A_i = (\widetilde{R}(\tau_i) - \mathrm{mean}_j\,\widetilde{R}(\tau_j)) / (\mathrm{std}_j\,\widetilde{R}(\tau_j)+\varepsilon)$. This makes the sparse terminal signal usable in low-resource regimes by supplementing it with dense structural feedback.

\textbf{Policy Update.}
We optimize a step-level KL-regularized group-relative objective:
\begin{equation}
\mathcal L_{\mathrm{SAGE}}(\theta) = \tfrac{1}{G}\sum_{i=1}^{G} \tfrac{1}{T_i}\sum_{t=1}^{T_i} \left[ \min\!\left\{ \rho_{i,t}(\theta)A_i, \mathrm{clip}(\rho_{i,t}(\theta),1-\epsilon,1+\epsilon)A_i \right\} - \delta \log \tfrac{\pi_\theta(a_{i,t}\mid s_{i,t})}{\pi_{\mathrm{ref}}(a_{i,t}\mid s_{i,t})} \right],
\end{equation}
with $\rho_{i,t}(\theta) = \pi_\theta(a_{i,t}\mid s_{i,t}) / \pi_{\theta_{\mathrm{old}}}(a_{i,t}\mid s_{i,t})$. For concentration analysis, we use the trajectory-level distribution
$
P_{\mathrm{EBM}}(\tau) = \pi_{\theta_{\mathrm{old}}}(\tau) \exp(\lambda\Psi_{\mathrm{SAGE}}(\tau)) / Z_\lambda, \Psi_{\mathrm{SAGE}}(\tau) = \sum_{t=1}^{T} \Psi_{\mathrm{SAGE}}(s_t,a_t).
$
\begin{proposition}[Guidance-Induced Feasible-Support Concentration]
\label{prop:\model{}_feasible_concentration}
Let $P_{\mathrm{EBM}}$ be the trajectory-level distribution, $\Psi_{\model}(\tau)=\sum_{t=1}^{T}\Psi_{\model}(s_t,a_t)$, and $\mathcal{F}$, $\mathcal{G}=\Omega\setminus\mathcal{F}$ the SCA-feasible and inadmissible regions. If there exists $\Delta>0$ with $\inf_{\tau\in\mathcal{F}}\Psi_{\model}(\tau) - \sup_{\tau\in\mathcal{G}}\Psi_{\model}(\tau) \ge \Delta$, then
$
P_{\mathrm{EBM}}(\mathcal{G}) / P_{\mathrm{EBM}}(\mathcal{F}) \le e^{-\lambda\Delta}\, \pi_{\theta_{\mathrm{old}}}(\mathcal{G}) / \pi_{\theta_{\mathrm{old}}}(\mathcal{F}).
$
Thus, increasing $\lambda$ exponentially suppresses locally inadmissible mass relative to feasible mass.
\end{proposition}

\section{Experiment}
\label{sec:exp}
We comprehensively evaluate \model{} across 12 diverse benchmarks, 10 backbone models and 4 competitive baselines. Our experiments address the following three questions:
\textbf{Q1:} \textit{Does \model{} improve reasoning performance across diverse benchmarks and model scales?}
\textbf{Q2:} \textit{Does \model{} alleviate exploration and compounding biases on long-horizon reasoning tasks?}
\textbf{Q3:} \textit{Do algebraic sparsification and hyperbolic structural guidance each contribute to the gains predicted by SCA?}

\Cref{sec:setup} describes the experimental setup. \Cref{sec:q1} reports performance on mathematical and free-form natural reasoning benchmarks. \Cref{sec:q3} evaluates long-horizon symbolic reasoning as a stress test for exploration and compounding biases. \Cref{sec:ablation} is the ablation study of the two components. We report implementation details and additional results in Appendix~\ref{app:exp}.

\subsection{Experimental Settings}
\label{sec:setup}
\textbf{Models.}
We evaluate \model{} across multiple model backbones, including Qwen3.5 (2B, 9B, 35B)~\citep{team2026qwen3}, Qwen3.6-27B~\citep{team2026qwen3}, Qwen3-32B~\citep{yang2025qwen3}, DeepSeekMath-7B~\citep{shao2024deepseekmath}, DeepSeek-Prover-V2-7B~\citep{ren2025deepseekprover}, Kimina-Prover (7B, Distill-8B)~\citep{wang2025kiminaprover}, Llama-3.3-70B-Instruct~\citep{patterson2022carbon}, MiniMax M2.5~\citep{chen2026minimax}, GLM-5~\citep{zeng2026glm} and Kimi K2.6~\citep{team2026kimi}.

\textbf{Benchmarks.}
We evaluate across three families of tasks: (i) \emph{Closed-form Mathematical Reasoning: } MATH~\citep{lightman2023lets}, Minerva Math~\citep{lewkowycz2022}, AMC23~\citep{mathai_amc23_2025}, AIME 2024~\citep{huggingfaceh4_aime2024_2025}, OlympiadBench~\citep{he2024olympiadbench}, GSM8K~\citep{cobbe2021training}, and Putnam~\citep{tsoukalas2024putnam}; (ii) \emph{Free-form Natural Reasoning: } MMLU-Pro~\citep{wang2024mmlupro}, GPQA~\citep{rein2023gpq}, BBH-H~\citep{suzgun2022challenging}, and ARC-C~\citep{allenai:arc}; and (iii) \emph{Real-world Long-horizon  Reasoning: } AC problem task, for which we follow~\citet{shehper2025make} to construct 1190 AC presentations with $n \leq 7$ and $\lvert w \rvert \leq 7$.

\textbf{Baselines.}
We compare with 4 representative baselines, including supervised fine-tuning (SFT),  GRPO~\citep{shao2024deepseekmath}, EMPO~\citep{zhang2025right} and GRPO-PRM~\citep{sullivan2025grpo}.
All methods use comparable rollout budgets, generation lengths, and decoding constraints.

\begin{table}[h!]
\vspace{-10pt}
\centering
\setlength{\tabcolsep}{3pt}
\caption{Accuracy (\%) on mathematical reasoning benchmarks. Best in \textbf{bold}, second best \underline{underlined}.}
\label{tab:math_full}
\resizebox{\textwidth}{!}{
\begin{tabular}{l c c c c c c c c}
\toprule
\textbf{Model} & \textbf{MATH} & \textbf{Minerva} & \textbf{Olympiad} & \textbf{AIME24} & \textbf{AMC23} & \textbf{GSM8K} & \textbf{Putnam} & \textbf{Avg.} \\
\midrule
\multicolumn{9}{l}{\emph{Flagship}} \\
Llama-3.3-70B-Instruct & 66.08 & 33.61 & 32.94 & 17.31 & 29.02 & 80.47 & 9.91 & 38.48 \\
MiniMax M2.5 & 82.43 & 49.76 & 65.71 & 64.28 & 69.43 & 93.54 & 20.84 & 63.71\\
GLM-5 & 83.02 & 50.51 & 68.74 & 66.08 & 70.02 & 94.01 & 21.08 & 64.78\\
Kimi K2.6 & 83.45 & 50.82 & 67.85 & 65.22 & 70.28 & 94.07 & 21.91 & 64.80\\
\midrule
\multicolumn{9}{l}{\emph{2B models}} \\
Qwen3.5           & 50.64 & 11.62 & 24.58 & 9.41  & 43.36 & 47.38 & 3.66 & 27.24 \\
Qwen3.5 w/SFT     & 60.41 & 26.52 & 27.96 & 3.74  & 37.88 & 55.06 & 4.63 & 30.89 \\
Qwen3.5 w/GRPO    & \underline{73.39} & \underline{33.27} & 33.86 & \textbf{16.05} & 50.94 & 63.12 & 6.79 & 39.63 \\
Qwen3.5 w/EMPO    & 71.94 & 31.39 & \underline{37.04} & 12.88 & \underline{54.71} & \underline{66.54} & \underline{7.48} & \underline{40.28} \\
\textbf{Qwen3.5 w/\model{}} & \textbf{74.61} & \textbf{34.02} & \textbf{38.25} & \underline{15.72} & \textbf{55.81} & \textbf{67.06} & \textbf{9.31} & \textbf{42.11} \\
\midrule
\multicolumn{9}{l}{\emph{9B models}} \\
Qwen3.5           & 65.11 & 14.34 & 27.31 & 6.38  & 39.73 & 45.59 & 4.61 & 29.01 \\
Qwen3.5 w/SFT     & 77.48 & 29.73 & \underline{40.15} & \underline{24.20} & 62.93 & \underline{70.91} & \underline{9.12} & \underline{44.93} \\
Qwen3.5 w/GRPO    & 75.96 & \underline{40.41} & 38.82 & 19.51 & 56.96 & 62.77 & 7.66 & 43.16 \\
Qwen3.5 w/EMPO    & \underline{78.24} & 39.27 & 37.03 & 20.88 & \textbf{64.88} & 65.55 & 8.01 & 44.84 \\
\textbf{Qwen3.5 w/\model{}} & \textbf{79.97} & \textbf{41.38} & \textbf{41.59} & \textbf{25.58} & \underline{63.91} & \textbf{71.72} & \textbf{11.47} & \textbf{47.95} \\
\midrule
\multicolumn{9}{l}{\emph{35B models}} \\
Qwen3.5              & 70.28 & 32.36 & 48.81 & 35.05 & 42.33 & 76.28 & 11.36 & 45.21 \\
Qwen3.5 w/SFT        & 74.69 & 38.57 & 52.96 & 41.84 & 49.18 & 82.73 & 14.67 & 50.66 \\
Qwen3.5 w/GRPO       & \underline{82.18} & 48.07 & \underline{65.31} & 57.94 & \underline{68.92} & 91.51 & 19.48 & 61.92 \\
Qwen3.5 w/EMPO       & 81.84 & \underline{48.39} & 64.41 & \textbf{62.31} & 67.38 & \underline{92.29} & \underline{19.97} & \underline{62.37} \\
\textbf{Qwen3.5 w/\model{}} & \textbf{84.72} & \textbf{51.34} & \textbf{68.26} & \underline{62.04} & \textbf{70.87} & \textbf{94.16} & \textbf{22.62} & \textbf{64.86} \\
\bottomrule
\end{tabular}}
\vspace{-5pt}
\end{table}

\subsection{Main Results}
\label{sec:q1}
\Cref{tab:math_full} shows \model{} consistently improves outcome-level accuracy across scales without gold reasoning traces or process labels. At the 2B, 9B, and 35B scales, average accuracy improves from $27.24\%$ to $42.11\%$, $29.01\%$ to $47.95\%$, and $45.21\%$ to $64.86\%$, surpassing the strongest baseline by $+1.83$, $+3.02$, and $+2.49$ points respectively. Notably, the 35B \model{} variant ($64.86\%$) substantially exceeds the Llama-3.3-70B-Instruct flagship ($38.48\%$) at roughly half the parameters, with consistent gains on the hardest competition-style benchmarks (Olympiad, AIME24, Putnam).

\textbf{Free-form Natural Reasoning.} The benefits generalize beyond structured formal domains (\Cref{tab:natural_reasoning}). At 9B, \model{} raises MMLU-Pro average from $37.91\%$ to $39.99\%$, BBH-H from $44.04\%$ to $45.31\%$, and ARC-C from $39.73\%$ to $42.04\%$; similar scaling holds at 27B (BBH-H $60.25\%$, ARC-C $58.11\%$) and 35B (BBH-H $69.07\%$, ARC-C $66.41\%$), all leading among question-only post-training methods. Five-seed mean-std results are in Appendix~\ref{app:addition}.

\subsection{Long-Horizon Reasoning} 
\label{sec:q3}
We evaluate the AC problem with two metrics~\citep{yang2023lean,hsiang2025leandojo}: \emph{AC Validity} (percentage of syntactically and logically valid steps, a proxy for local precision) and \emph{Lean-Verified Proofs} (success rate of compiler-checked proofs, the gold standard for end-to-end rigor). As shown in \Cref{fig:ac}, \model{} consistently outperforms baselines on both: AC Validity gains of $+19.2\%$ to $+26.0\%$ across architectures, and Lean-Verified gains over $+13\%$ in every case and Qwen3 reaches a nearly 8-fold increase over base. These results are consistent with the SCA prediction that controlling exploration and compounding biases yields more stable long-horizon trajectories.

\begin{table}[h!]
\vspace{-10pt}
\centering
\setlength{\tabcolsep}{3pt}
\caption{
Accuracy (\%) on free-form natural reasoning benchmarks. The best is in \textbf{bold} with second best in \underline{underline}.
}
\label{tab:natural_reasoning}
\resizebox{0.9\textwidth}{!}{
\begin{tabular}{l c c c c c c c c}
\toprule
\textbf{Model} & \textbf{STEM}
& \multicolumn{4}{c}{\textbf{MMLU-Pro}} 
& \textbf{GPQA} 
& \textbf{BBH-H} 
& \textbf{ARC-C} \\
\cmidrule(lr){3-6}
 &  &  \footnotesize{Humanity} & \footnotesize{Social} & Other & Avg. &  &  &  \\
\midrule
\multicolumn{9}{l}{\emph{9B models}} \\
Qwen3.5            & 12.71 & 8.02  & 14.95 & 10.21 & 11.06 & 10.91 & 21.58 & 18.49 \\
Qwen3.5 w/SFT      & 20.18 & 11.36 & 28.97 & 19.22 & 19.85 & 12.31 & 32.44 & 29.56 \\
Qwen3.5 w/GRPO     & \underline{33.38} & \underline{28.31} & \underline{50.41} & \underline{39.26} & \underline{39.33} & 18.21 & 41.69 & 37.82 \\
Qwen3.5 w/EMPO     & 32.57 & 27.32 & 48.73 & 37.69 & 37.91 & \textbf{21.11} & \underline{44.04} & \underline{39.73} \\
\textbf{Qwen3.5 w/\model{}} & \textbf{34.11} & \textbf{29.08} & \textbf{51.17} & \textbf{39.71} & \textbf{39.99} & \underline{20.86} & \textbf{45.31} & \textbf{42.04} \\
\midrule
\multicolumn{9}{l}{\emph{27B models}} \\
Qwen3.6            & 30.29 & 24.34 & 46.55 & 35.31 & 35.40 & 16.09 & 38.70 & 34.78 \\
Qwen3.6 w/SFT      & 34.18 & 28.51 & 41.52 & 37.28 & 35.77 & 22.67 & 45.27 & 41.12 \\
Qwen3.6 w/GRPO     & \textbf{57.74} & \underline{37.02} & \textbf{65.16} & \underline{57.55} & \underline{53.24} & \textbf{34.29} & 55.74 & 51.78 \\
Qwen3.6 w/EMPO     & 53.96 & 35.58 & 60.04 & 52.18 & 49.27 & 29.52 & \underline{57.19} & \underline{53.26} \\
\textbf{Qwen3.6 w/\model{}} & \underline{56.31} & \textbf{37.91} & \underline{64.43} & \textbf{58.25} & \textbf{53.53} & \underline{32.51} & \textbf{60.25} & \textbf{57.11} \\
\midrule
\multicolumn{9}{l}{\emph{35B models}} \\
Qwen3.5              & 45.08 & 36.31 & 52.29 & 44.28 & 44.29 & 31.05 & 47.52 & 45.37 \\
Qwen3.5 w/SFT        & 49.84 & 38.26 & 54.47 & 48.62 & 47.12 & 28.97 & 52.96 & 49.18 \\
Qwen3.5 w/GRPO       & \underline{63.58} & \underline{43.19} & \underline{69.08} & \underline{60.71} & \underline{57.66} & \underline{35.78} & 63.29 & 60.91 \\
Qwen3.5 w/EMPO       & 61.96 & 42.02 & 68.61 & 59.54 & 56.72 & 35.43 & \underline{65.02} & \underline{62.97} \\
\textbf{Qwen3.5 w/\model{}} & \textbf{65.27} & \textbf{44.51} & \textbf{71.22} & \textbf{62.25} & \textbf{59.33} & \textbf{38.58} & \textbf{69.07} & \textbf{66.41} \\
\bottomrule
\end{tabular}
}
\vspace{-5pt}
\end{table}

\begin{figure}[h!]
\vspace{-2pt}
    \centering    \includegraphics[width=1\linewidth]{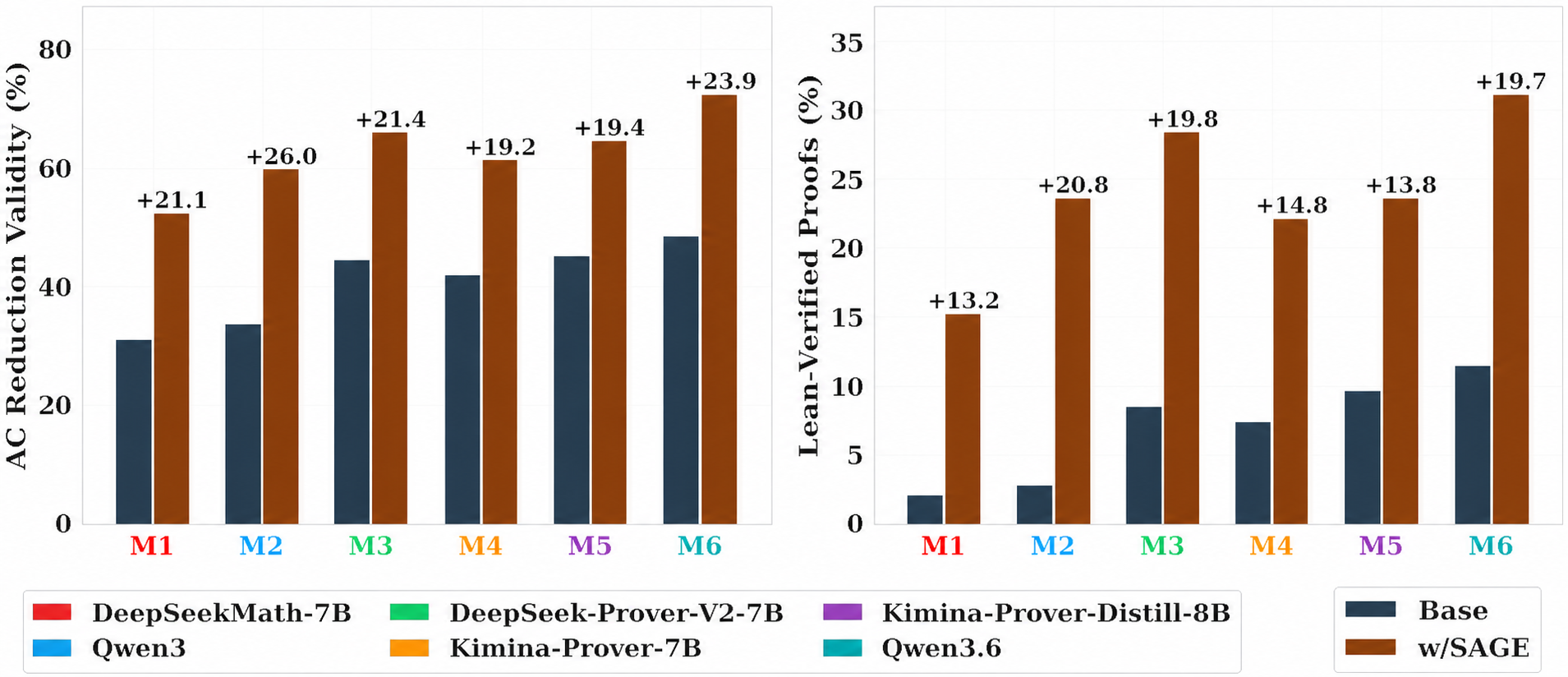}
    \caption{Comparative performance of models with \model{} versus base models across two primary metrics: AC Validity (Left) and Lean-Verified Proofs (Right). Numbers annotated above bars indicate the absolute percentage point improvement.}
    \vspace{-5pt}
    \label{fig:ac}
\end{figure}

\subsection{Component Ablation Analysis}
\label{sec:ablation}

We isolate the two guidance components: $\Psi_{\mathcal P}$ (algebraic sparsification, controlling exploration bias) and $\Psi_{\mathcal H}$ (hyperbolic structural guidance, mitigating compounding bias). On Qwen3.5-9B mathematical and free-form reasoning, removing either degrades performance (Table~\ref{tab:ablations_core_qwen7b}), showing the two signals are complementary rather than redundant. Two controls rule out a generic dense-shaping explanation: replacing $\Psi_{\mathcal H}$ with Euclidean distance weakens performance and shuffling target anchors substantially reduces both accuracy and reward density.


\textbf{Comparison with learned process rewards.} 
Table~\ref{tab:prm_main} shows that GRPO-PRM improves over GRPO, confirming that dense process feedback helps under sparse rewards. However, \model{} remains 
\begin{wraptable}{r}{0.5\linewidth}
\vspace{5pt}
\centering
\small
\setlength{\tabcolsep}{3pt}
\renewcommand{\arraystretch}{1.05}
\vspace{-8pt}
\caption{Core ablations on Qwen3.5-9B. All variants share the same entropy-filtered training subset, rollout budget, decoding constraints, optimization steps, and KL coefficient.}
\label{tab:ablations_core_qwen7b}
\begin{tabular}{@{}lcccc@{}}
\toprule
\textbf{Variant} & \multicolumn{2}{c}{\textbf{Olympiad}} & \multicolumn{2}{c}{\textbf{BBH-H}} \\
\cmidrule(lr){2-3} \cmidrule(lr){4-5}
 & \textbf{Acc} & \textbf{Rew./1k} & \textbf{Acc} & \textbf{Rew./1k} \\
\midrule
GRPO & 38.82 & 16.08 & 41.69 & 20.06 \\
EMPO & 37.03 & 15.03 & 44.04 & 19.02 \\
\midrule
\model{} w/o $\Psi_{\mathcal H}$ & 39.84 & 18.47 & 39.06 & 22.36 \\
\model{} w/o $\Psi_{\mathcal P}$ & 39.12 & 17.68 & 38.85 & 21.18 \\
\model{} w/ Euclidean            & 38.01 & 16.92 & 38.02 & 20.89 \\
\model{} w/ Shuffled $g$         & 37.99 & 17.01 & 37.83 & 19.91 \\
\model                            & \textbf{41.59} & \textbf{21.38} & \textbf{45.31} & \textbf{25.76} \\
\bottomrule
\end{tabular}
\end{wraptable}
stronger, especially on the AC problem, which demonstrates that the gains are thus not explained by generic dense reward shaping alone, but by target-aligned structural guidance.


\textbf{Direct bias ablation on AC task.}
To directly test whether the two components mitigate the symbolic long-horizon failure modes predicted by SCA, we repeat the ablation on AC task. As shown in Table~\ref{tab:ac_component_ablation_main}, removing either $\Psi_{\mathcal P}$ or $\Psi_{\mathcal H}$ reduces AC Validity, AC Path Solving, and Lean-Verified success, while full \model{} achieves the strongest end-to-end verified performance-confirming that both signals are needed when structural validity and long-horizon extension are jointly required.

\begin{table}[h!]
\vspace{10pt}
\centering
\caption{Left: comparison with a learned process-reward baseline (Qwen3.5-35B). Right: component ablation on AC task (Qwen3.5-35B).}
\label{tab:combined_35b}
\begin{minipage}{0.46\linewidth}
\centering
\small
\setlength{\tabcolsep}{4pt}
\subcaption{PRM baseline comparison.}
\label{tab:prm_main}
\begin{tabular}{@{}lccc@{}}
\toprule
Method & Olym. & BBH-H & Lean \\
\midrule
GRPO     & 65.31 & 63.29 & 14.64 \\
GRPO-PRM & 63.27 & 59.12 & 17.36 \\
\model{} & \textbf{68.26} & \textbf{69.07} & \textbf{23.69} \\
\bottomrule
\end{tabular}
\end{minipage}\hfill
\begin{minipage}{0.52\linewidth}
\vspace{-5pt}
\centering
\small
\setlength{\tabcolsep}{3pt}
\subcaption{AC component ablation.}
\label{tab:ac_component_ablation_main}
\begin{tabular}{@{}lccc@{}}
\toprule
Variant & AC Valid. & AC Path & Lean \\
\midrule
GRPO & 54.28 & 23.05 & 14.64 \\
EMPO & 53.18 & 21.52 & 13.78 \\
\model{} w/o $\Psi_{\mathcal H}$ & 57.02 & 27.14 & 18.82 \\
\model{} w/o $\Psi_{\mathcal P}$ & 56.31 & 25.98 & 18.07 \\
\model{}                          & \textbf{59.83} & \textbf{31.76} & \textbf{23.69} \\
\bottomrule
\end{tabular}
\end{minipage}
\end{table}


\section{Related Work}
\label{sec:rel}

\textbf{Reinforcement Learning and Supervision for Reasoning.}
Outcome-based RL has become a dominant paradigm for improving LLM reasoning \citep{ouyang2022training, shao2024deepseekmath, yu2025dapo}, yet its effectiveness degrades as reasoning horizon $T$ grows and terminal feedback becomes sparse \citep{suo2025long, lancontext2025, lan2026alternating}. Process-level supervision via verifiers, formal proofs, or step-wise reward models \citep{yang2023lean, lightman2023lets, yao2024swift} mitigates this issue but shifts the bottleneck to annotation cost and verifier coverage.

\textbf{Exploration Strategies and Optimization Biases.}
A parallel line of work improves exploration under sparse rewards through entropy regularization, syntactic constraints, empirical consistency, and self-improvement objectives \citep{zhang2025right, she2025rprm, zhang2025count, yue2025does, gai2025differential, lin2026cec}. Recent studies further reveal that KL-regularized RL can induce mode collapse through the optimization objective itself \citep{anthony2510kl}, while latent Poincar\'e shaping leverages hyperbolic geometry to provide dense process rewards and guide search \citep{xia2026latent}. Despite these advances, how logical constraints shape the admissible trajectory space in long-horizon reasoning, and how such structure can be exploited for principled exploration, remains largely underexplored \citep{wu2024ocean}.

\section{Conclusion}
\label{sec:con}
In this work, we study long-horizon reasoning under sparse and delayed rewards, showing that its failures arise not only from task difficulty but from two structural biases in outcome-based post-training: exploration bias and compounding bias. We introduce SCA as a theoretical lens for characterizing these biases through prefix-closed feasible regions and for identifying two mitigation requirements: feasible-support concentration and prefix-level structural correction.
Building on SCA, we propose \model{}, a topology-guided post-training framework that implements these requirements through algebraic sparsification and hyperbolic structural guidance. Across 12 benchmarks and 10 model backbones, \model{} consistently improves over strong post-training baselines. On the open real-world AC task, \model{} achieves a nearly 8-fold improvement. 

\section*{Acknowledgements}
We thank the anonymous reviewers for their constructive comments. This work is supported by the National Science Foundation under Award No. 2339989, No. 2449769, No. 2406439 and No. IIS 2543142; DARPA under contract No. HR00112490370 and No. HR001124S0013, U.S. Department of Homeland Security under Grant Award No. 17STCIN00001-08-00, Amazon-Virginia Tech Initiative for Efficient and Robust Machine Learning, Amazon AWS, Google, Cisco, 4-VA, Commonwealth Cyber Initiative, National Surface Transportation Safety Center for Excellence, and Virginia Tech. The views and conclusions are those of the authors and should not be interpreted as representing the official policies of the funding agencies or the government.

\newpage

\bibliographystyle{icml2026} 
\bibliography{reference}
\newpage

\appendix
\section{Formal Properties of Symbolic Closure Analysis}
\label{app:sca}

We restate the formal object used by Symbolic Closure Analysis (SCA).
SCA is an analytical lens for characterizing feasible-support concentration under local admissibility.
It is not itself an inference-time search procedure.

Let $\mathfrak S$ be a domain-specific symbolic interface that specifies admissible operators and a local admissibility predicate.
A trajectory $\tau=\{(v_t,op_t)\}_{t=1}^T$ belongs to the SCA-feasible region $\mathcal{F}$ iff every local transition is admissible:
\begin{equation}
\tau \in \mathcal{F}
\quad \Longleftrightarrow \quad
\forall t,\ 
\mathrm{Adm}_{\mathfrak S}\!\bigl(\rho(v_t),\, op_t,\, v_t\bigr)=1.
\label{eq:sca_def_app}
\end{equation}
Here $\rho(v_t)$ denotes the predecessor context of $v_t$, and $op_t$ is the operator used to obtain $v_t$.

\paragraph{Local admissibility is not process supervision.}
The predicate $\mathrm{Adm}_{\mathfrak S}$ checks only local well-formedness or rule consistency under the task interface $\mathcal{S}$.
It does not provide ground-truth solution steps, does not reveal successful trajectories, and does not use terminal outcome labels.
Thus, SCA separates local admissibility from global correctness.

\paragraph{Prefix closure.}
The feasible region $\mathcal{F}$ is prefix-closed by construction.
If a prefix violates local admissibility, then any continuation of that prefix remains outside $\mathcal{F}$.
Equivalently, if $\tau\in\mathcal{F}$, then every prefix $\tau_{\le t}$ also lies in $\mathcal{F}$.

\begin{lemma}[Prefix closure under local admissibility]
\label{lem:prefix_closure_app}
Let $\tau=\{(v_t,op_t)\}_{t=1}^T$ and suppose that for some step $t$,
\[
\mathrm{Adm}_{\mathfrak S}\!\bigl(\rho(v_t),op_t,v_t\bigr)=0.
\]
Then every trajectory $\tau'$ that contains the same prefix up to step $t$ satisfies $\tau'\notin\mathcal{F}$.
\end{lemma}

\begin{proof}
By Definition~\eqref{eq:sca_def_app}, membership in $\mathcal{F}$ requires every transition to satisfy local admissibility.
If the step-$t$ transition violates $\mathrm{Adm}_{\mathfrak S}$, then the conjunction in \eqref{eq:sca_def_app} fails.
Any continuation that preserves this invalid prefix also contains the same failed transition, and therefore cannot belong to $\mathcal{F}$.
\end{proof}

\paragraph{Feasible but unsuccessful trajectories.}
Local admissibility does not imply terminal success.
Let $\mathcal{T}^*\subseteq \Omega$ denote the set of globally successful trajectories, such as trajectories that produce a correct final answer or a compiler-checked proof.
We assume $\mathcal{T}^*\subseteq\mathcal{F}$, but $\mathcal{F}$ may contain many feasible-yet-unsuccessful trajectories.
Define
\[
\mathcal{B}:=\mathcal{F}\setminus\mathcal{T}^*.
\]
For a rollout distribution $\pi$ supported primarily on $\mathcal{F}$, define the success density inside the feasible region as
\[
p_{\mathcal{F}} :=
\Pr_{\tau\sim\pi}\!\bigl(\tau\in\mathcal{T}^*\mid \tau\in\mathcal{F}\bigr),
\]
and the feasible-but-failing mass as
\[
\beta_{\mathcal{F}} :=
\Pr_{\tau\sim\pi}\!\bigl(\tau\in\mathcal{B}\mid \tau\in\mathcal{F}\bigr)
=
1-p_{\mathcal{F}}.
\]
In intrinsically complex long-horizon tasks, $\beta_{\mathcal{F}}$ can remain large even when local admissibility is high.
This explains why improving step-level validity alone is insufficient: \model{} must also provide depth-wise structural guidance that helps feasible prefixes extend toward successful trajectories.

\paragraph{Relation to \model{}.}
\model{} implements the requirements identified by SCA through soft training-time guidance.
Rather than imposing a hard projection onto $\mathcal{F}$, \model{} uses structural potentials to bias rollout sampling and reward shaping:
\[
\Psi_{\mathrm{\model{}}}(s_t,a_t)
=
\alpha \Psi_P(s_t,a_t)
+
\gamma \Psi_H(s_t,a_t).
\]
This produces a training-time rollout distribution that increases probability mass on structurally informative trajectories while preserving direct inference with the trained policy.
In domains where an exact local checker is available, hard validity masking can be viewed as a limiting or diagnostic variant, but it is not required by the general \model{} framework and is not used at inference time.

\section{Proof of Theorem~\ref{thm:sca_kl_collapse}}
\label{app:proof_kl_collapse}

We prove the total-variation bound for the KL-regularized optimizer under sparse-reward regimes.

\thmScaKlCollapse*
\begin{proof}
The optimizer of the KL-regularized objective is the exponentially tilted
distribution
\[
\pi^*_{\Omega}(\tau)
=
\frac{
\pi_{\mathrm{ref}}(\tau)\exp(R(\tau)/\lambda)
}{
Z
},
\qquad
Z=
\mathbb{E}_{\tau\sim\pi_{\mathrm{ref}}}
\left[\exp(R(\tau)/\lambda)\right].
\]
Define
\[
h(\tau)=\exp(R(\tau)/\lambda)-1.
\]
Since $R(\tau)=0$ for $\tau\notin S$ and
$R(\tau)\in[0,R_{\max}]$, we have
\[
0\le h(\tau)\le e^{R_{\max}/\lambda}-1,
\qquad
h(\tau)=0\ \text{for}\ \tau\notin S.
\]
Let
\[
H=\mathbb{E}_{\pi_{\mathrm{ref}}}[h(\tau)].
\]
Then
\[
0\le H\le \bigl(e^{R_{\max}/\lambda}-1\bigr)p,
\qquad
Z=1+H.
\]
Therefore,
\[
\pi^*_{\Omega}(\tau)-\pi_{\mathrm{ref}}(\tau)
=
\pi_{\mathrm{ref}}(\tau)
\left(
\frac{1+h(\tau)}{1+H}-1
\right)
=
\pi_{\mathrm{ref}}(\tau)
\frac{h(\tau)-H}{1+H}.
\]
Hence
\[
D_{\mathrm{TV}}(\pi^*_{\Omega},\pi_{\mathrm{ref}})
=
\frac{1}{2}
\mathbb{E}_{\pi_{\mathrm{ref}}}
\left[
\frac{|h(\tau)-H|}{1+H}
\right]
\le
\frac{1}{2(1+H)}
\left(
\mathbb{E}_{\pi_{\mathrm{ref}}}[h(\tau)]
+
H
\right)
=
\frac{H}{1+H}
\le H.
\]
Using the bound on $H$ gives
\[
D_{\mathrm{TV}}(\pi^*_{\Omega},\pi_{\mathrm{ref}})
\le
\bigl(e^{R_{\max}/\lambda}-1\bigr)p.
\]
Finally, when $R_{\max}\ll\lambda$,
\[
e^{R_{\max}/\lambda}-1
=
\frac{R_{\max}}{\lambda}
+
O\!\left(\frac{R_{\max}^2}{\lambda^2}\right),
\]
which gives the stated asymptotic scaling.
\end{proof}

\section{Recovery Guarantee for the Greedy Sparse Locator}
\label{app:gsp}

This appendix provides a standard recovery guarantee for the greedy sparse locator used to instantiate algebraic sparsification in tangent-space coordinates.
The result supports the intuition that, when the residual admits a sparse decomposition over operator-indexed directions, greedy projection can identify structurally relevant operators under standard coherence conditions.

\subsection{Setup}

Let $D=[d_1,\dots,d_N]\in\mathbb{R}^{d\times N}$ be a dictionary with normalized atoms $\|d_j\|_2=1$.
Assume that a residual vector $r\in\mathbb{R}^d$ has a $k$-sparse representation
\[
r = D_I\alpha_I,
\]
where $I\subseteq[N]$ is the support with $|I|=k$.
Define the mutual coherence
\[
\mu := \max_{i\neq j} |\langle d_i,d_j\rangle|.
\]
The greedy sparse locator selects atoms by maximum correlation with the current residual and then orthogonally projects, matching the classical Orthogonal Matching Pursuit (OMP) procedure.

\begin{theorem}[OMP support recovery under coherence~\citep{tropp2007}]
\label{thm:omp_app}
Let $D=[d_1,\ldots,d_N]\in\mathbb{R}^{m\times N}$ be a dictionary with
unit-norm columns, and let
\[
\mu(D)=\max_{p\neq q}|\langle d_p,d_q\rangle|
\]
denote its mutual coherence. Suppose $y=D\alpha$ is noiseless and
$\alpha$ is $k$-sparse with support $I$. If
\[
\mu(D)<\frac{1}{2k-1},
\]
then OMP run for $k$ iterations recovers the exact support $I$.
\end{theorem}

\begin{proof}
We use the standard exact recovery condition (ERC) for OMP. For a fixed
support $I$, OMP exactly recovers every signal supported on $I$ in the
noiseless setting if
\[
\max_{j\notin I}\|D_I^\dagger d_j\|_1 < 1,
\]
where $D_I$ is the subdictionary indexed by $I$ and
$D_I^\dagger=(D_I^\top D_I)^{-1}D_I^\top$.

It remains to show that the mutual coherence condition implies this ERC.
Let
\[
G_I = D_I^\top D_I.
\]
Since the columns of $D$ are normalized, $G_I$ has diagonal entries equal
to $1$ and off-diagonal entries bounded in absolute value by $\mu(D)$.
Hence
\[
\|I-G_I\|_1 \le (k-1)\mu(D).
\]
Under $\mu(D)<1/(2k-1)$, we have $(k-1)\mu(D)<1$, so $G_I$ is invertible,
and the Neumann-series bound gives
\[
\|G_I^{-1}\|_1
\le
\frac{1}{1-(k-1)\mu(D)}.
\]
For any $j\notin I$,
\[
\|D_I^\top d_j\|_1
\le
k\mu(D).
\]
Therefore,
\[
\|D_I^\dagger d_j\|_1
=
\|(D_I^\top D_I)^{-1}D_I^\top d_j\|_1
\le
\frac{k\mu(D)}{1-(k-1)\mu(D)}.
\]
The condition $\mu(D)<1/(2k-1)$ is equivalent to
\[
\frac{k\mu(D)}{1-(k-1)\mu(D)}<1.
\]
Thus the ERC holds. By the standard OMP exact recovery theorem
\citep{tropp2007}, OMP selects atoms from the true support at every
iteration and recovers $I$ after $k$ iterations.
\end{proof}

The theorem is not a guarantee that \model{} globally solves the reasoning task.
It only justifies the algebraic sparsification step under a standard sparse-residual model:
when the unresolved residual is concentrated on a small number of operator-aligned directions, greedy projection can identify the relevant structural directions.
This supports using $\Psi_P$ as a soft compatibility score for feasible-support concentration.

\section{Soft Feasible-Support Concentration}
\label{app:sage_soft}

This section analyzes an idealized trajectory-level reweighting induced by
\model{}. The result complements Proposition~\ref{prop:\model{}_feasible_concentration} in the main paper and
formalizes how a structural potential suppresses locally inadmissible rollout
mass when it separates feasible and infeasible trajectories. This is a
conditional concentration result, not a universal guarantee.

\paragraph{Trajectory-level reweighting.}
Let $\mathcal{T}$ denote the discrete set of trajectories and let
$\pi_0(\tau\mid q)$ be a base rollout distribution. We consider the
trajectory-level reweighted distribution
\begin{equation}
\pi_\lambda(\tau\mid q)
=
\frac{
\pi_0(\tau\mid q)
\exp\!\bigl(\lambda \Psi(\tau,q)\bigr)
}{
Z_\lambda(q)
},
\qquad
Z_\lambda(q)
=
\sum_{\tau\in\mathcal{T}}
\pi_0(\tau\mid q)
\exp\!\bigl(\lambda \Psi(\tau,q)\bigr),
\label{eq:sage_reweight_app}
\end{equation}
where $\Psi(\tau,q)$ denotes the trajectory-level \model{} potential and
$\lambda\ge 0$ controls guidance strength.

\paragraph{Feasible and infeasible sets.}
Let
\[
\mathcal{T}_{\mathrm{bad}}(q)
:=
\{\tau\in\mathcal{T}:\tau\notin\mathcal{F}(q)\},
\qquad
\mathcal{T}_{\mathrm{good}}(q)
:=
\mathcal{T}\setminus\mathcal{T}_{\mathrm{bad}}(q).
\]
Define the base invalid mass as
\[
p_0(q)
:=
\Pr_{\tau\sim\pi_0}
\bigl(\tau\in\mathcal{T}_{\mathrm{bad}}(q)\bigr).
\]

\paragraph{Relative margin condition.}
Assume that the structural potential separates feasible and infeasible
trajectories by a relative gap:
\begin{equation}
\inf_{\tau\in\mathcal T_{\mathrm{good}}(q)}
\Psi(\tau,q)
-
\sup_{\tau\in\mathcal T_{\mathrm{bad}}(q)}
\Psi(\tau,q)
\ge m.
\label{eq:sage_relative_margin_app}
\end{equation}

\begin{proposition}[Soft suppression under a relative structural margin]
\label{prop:sage_soft_suppression_app}
Under \Cref{eq:sage_reweight_app,eq:sage_relative_margin_app},
\[
\Pr_{\tau\sim\pi_\lambda}
\bigl(\tau\in\mathcal T_{\mathrm{bad}}(q)\bigr)
\le
\frac{
p_0(q)e^{-\lambda m}
}{
1-p_0(q)+p_0(q)e^{-\lambda m}
}.
\]
\end{proposition}

\begin{proof}
Let
\[
b(q)
=
\sup_{\tau\in\mathcal T_{\mathrm{bad}}(q)}
\Psi(\tau,q),
\qquad
\widetilde{\Psi}(\tau,q)
=
\Psi(\tau,q)-b(q).
\]
This additive shift does not change $\pi_\lambda$, because the factor
$\exp(-\lambda b(q))$ cancels between the numerator and the normalizing
constant. By the relative margin assumption,
\[
\widetilde{\Psi}(\tau,q)\le 0
\quad
\text{for } \tau\in\mathcal T_{\mathrm{bad}}(q),
\]
and
\[
\widetilde{\Psi}(\tau,q)\ge m
\quad
\text{for } \tau\in\mathcal T_{\mathrm{good}}(q).
\]
Therefore,
\[
\sum_{\tau\in\mathcal T_{\mathrm{bad}}(q)}
\pi_0(\tau\mid q)
\exp(\lambda\widetilde{\Psi}(\tau,q))
\le
p_0(q),
\]
while
\[
\sum_{\tau\in\mathcal T_{\mathrm{good}}(q)}
\pi_0(\tau\mid q)
\exp(\lambda\widetilde{\Psi}(\tau,q))
\ge
(1-p_0(q))e^{\lambda m}.
\]
Hence
\[
\Pr_{\pi_\lambda}
\bigl(\mathcal T_{\mathrm{bad}}(q)\mid q\bigr)
\le
\frac{
p_0(q)
}{
p_0(q)+(1-p_0(q))e^{\lambda m}
}.
\]
Multiplying the numerator and denominator by $e^{-\lambda m}$ gives
\[
\Pr_{\pi_\lambda}
\bigl(\mathcal T_{\mathrm{bad}}(q)\mid q\bigr)
\le
\frac{
p_0(q)e^{-\lambda m}
}{
1-p_0(q)+p_0(q)e^{-\lambda m}
}.
\]
\end{proof}

The bound shows that the infeasible mass is suppressed exponentially in the
guidance strength $\lambda$ when the structural potential separates feasible
and infeasible trajectories by a relative margin. The result does not require
bad trajectories to receive negative potential values; it is invariant to
additive shifts of $\Psi$.

\section{Experimental Details}
\label{app:exp}

This appendix provides additional implementation details for \model{}. We describe the
task-specific instantiation of the structural potentials, the rollout sampling and reward
construction used during post-training, and the auxiliary probes used to construct structural
signals in non-symbolic tasks. All structural modules described below are used only during
post-training. At inference time, we use the trained policy directly without evaluating
$\Psi_P$, $\Psi_H$, residual probes, semantic clusters, or local checkers.

\subsection{Task-Specific Instantiation}
\label{app:task_instantiation}

\model{} instantiates the two structural requirements identified by SCA: feasible-support
concentration for mitigating exploration bias and prefix-level structural correction for
mitigating compounding bias. The concrete implementation depends on the task interface.

\paragraph{AC Task.}
The AC task setting provides an explicit local admissibility interface. The state
$s_t$ is the current group presentation after the generated prefix, and the action $a_t$ is an
AC move. The residual $r_t$ represents unresolved algebraic structure in the
current presentation, including generator counts, relator lengths, and unresolved relator
components. The target anchor $g$ is the task-specified trivial presentation. In this setting,
algebraic sparsification scores the compatibility between a candidate move and the unresolved
algebraic residual, while hyperbolic structural guidance measures progress toward the target
presentation in a geometry suited to tree-like long-horizon transformations. This is the domain
where SCA gives an exact symbolic instantiation through the local admissibility interface.

\paragraph{Closed-form mathematical reasoning.}
For mathematical reasoning, we do not claim an exact SCA guarantee. Instead, \model{} uses the
SCA requirements as design principles for learned training-time structural priors. The state is
the current solution prefix concatenated with the original problem. The residual $r_t$ is estimated
from unresolved symbolic and semantic constraints extracted from the problem statement and
the current prefix. These constraints include problem entities, variables, equation structure,
operation category, and answer-schema status. The target anchor $g$ is prompt-conditioned and
derived from the expected constraint structure of the task, not from gold solution traces, gold
rationales, test labels, or terminal correctness labels.

\paragraph{Free-form natural reasoning.}
For free-form reasoning, the state is the partial rationale or response prefix concatenated with
the original prompt. The residual $r_t$ captures unresolved prompt requirements and prefix-level
semantic structure. The target anchor $g$ is a prompt-conditioned semantic anchor estimated from
training-rollout representations. It is not derived from test labels, gold answers, gold rationales,
or ground-truth solution traces. We instantiate SCA's structural principles via learned proxies; their fidelity to the symbolic guarantees is empirical.

For non-symbolic reasoning tasks, SAGE does not normalize over the full
token vocabulary or the full space of textual continuations. During
post-training, each action is a candidate reasoning step proposed by the
old policy. We sample a finite candidate set and reweight candidates using
length-normalized policy log-probability and structural potentials. This
finite-candidate procedure is used only for rollout generation during
training. At inference time, all models decode directly from the trained
policy without structural scoring or candidate reweighting.

To isolate the effect of structural guidance from prompt-selection effects,
all ablations and dense-reward comparisons use a fixed entropy-filtered
training subset constructed once from reference-policy rollouts. GRPO, EMPO,
PRM-GRPO, and SAGE are trained on the same retained prompts under the same
rollout budgets and optimization schedule. All test results are computed on
the complete benchmark test sets.

\subsection{Training-Time Rollout Sampling and Reward Construction}
\label{app:rollout_reward}

For symbolic Andrews-Curtis reasoning, an action $a_t$ is a discrete
Andrews-Curtis move provided by the task interface. For closed-form
mathematical reasoning and free-form natural reasoning, an action $a_t$
denotes a coarse candidate reasoning step rather than a single token. In
practice, a candidate step is generated by $\pi_{\theta_{\mathrm{old}}}$
until a step delimiter, sentence boundary, final-answer marker, end-of-sequence
token, or a maximum step length $L_{\mathrm{step}}$ is reached.

At each state $s_t$, we first sample a finite candidate set
\[
\mathcal{C}_t
=
\{a_t^{(1)},\ldots,a_t^{(K)}\}
\]
from $\pi_{\theta_{\mathrm{old}}}(\cdot\mid s_t)$. We then evaluate the
structural potentials only on this finite candidate set and sample the next
step according to
\[
P_{\mathrm{sample}}(a_t^{(k)}\mid s_t,\mathcal{C}_t)
=
\frac{
\exp\!\left(
\bar{\ell}_{\theta_{\mathrm{old}}}(a_t^{(k)}\mid s_t)
+
\lambda \Psi_{\model{}}(s_t,a_t^{(k)})
\right)
}{
\sum_{k'=1}^{K}
\exp\!\left(
\bar{\ell}_{\theta_{\mathrm{old}}}(a_t^{(k')}\mid s_t)
+
\lambda \Psi_{\model{}}(s_t,a_t^{(k')})
\right)
},
\]
where
\[
\bar{\ell}_{\theta_{\mathrm{old}}}(a_t^{(k)}\mid s_t)
=
\frac{1}{|a_t^{(k)}|}
\sum_{u=1}^{|a_t^{(k)}|}
\log \pi_{\theta_{\mathrm{old}}}
\bigl(a_{t,u}^{(k)}\mid s_t,a_{t,<u}^{(k)}\bigr)
\]
is the length-normalized log-probability of the candidate step. Length
normalization prevents the finite-candidate reweighting rule from favoring
shorter steps solely because of token-product probability effects.

The SAGE potential is
\[
\Psi_{\model{}}(s_t,a_t)
=
\alpha \Psi_P(s_t,a_t)
+
\gamma \Psi_H(s_t,a_t).
\]
This finite-candidate sampler is used only during post-training rollout
generation. It is not an exact normalization over the full space of textual
actions or the full token vocabulary. We use the AdamW optimizer~\citep{loshchilov2017decoupled} if applicable.  At inference time, we use the trained
policy directly without candidate-step reweighting or structural scoring.

\begin{algorithm*}[h!]
\caption{SAGE-Guided Policy Optimization}
\label{alg:sage_training}
\begin{algorithmic}[1]
\Require Training prompts $\mathcal D_{\mathrm{filt}}$,
old policy $\pi_{\theta_{\mathrm{old}}}$,
reference policy $\pi_{\mathrm{ref}}$,
frozen structural potentials $\Psi_P,\Psi_H$,
terminal reward $R_{\mathrm{term}}$,
weights $\alpha,\gamma,\lambda,\eta$

\For{each training minibatch $\mathcal B
\subseteq\mathcal D_{\mathrm{filt}}$}
  \For{each problem $q\in\mathcal B$}
    \For{$i=1,\ldots,G$}
      \State Initialize state $s_0$ from $q$
      \For{$t=0,\ldots,T-1$ until termination}
        \State Sample $K$ candidate steps
        $C_t\sim\pi_{\theta_{\mathrm{old}}}(\cdot\mid s_t)$
        \For{each $a\in C_t$}
          \State Compute $\Psi_P(s_t,a)$ and $\Psi_H(s_t,a)$
          \State $\Psi_{\mathrm{SAGE}}(s_t,a)
          \gets\alpha\Psi_P(s_t,a)+\gamma\Psi_H(s_t,a)$
          \State $u_t(a)\gets
          \bar{\ell}_{\theta_{\mathrm{old}}}(a\mid s_t)
          +\lambda\Psi_{\mathrm{SAGE}}(s_t,a)$
        \EndFor
        \State Sample $a_t$ from $C_t$ with probabilities
        proportional to $\exp(u_t(a))$
        \State $s_{t+1}\gets s_t\circ a_t$
      \EndFor
      \State Construct trajectory $\tau_i$
      \State $\widetilde R(\tau_i)\gets
      R_{\mathrm{term}}(\tau_i)
      +\frac{\eta}{T_i}
      \sum_{t=0}^{T_i-1}
      \Psi_{\mathrm{SAGE}}(s_t,a_t)$
    \EndFor
    \State Compute group-relative advantages
    \State $A_i\gets
    \frac{
    \widetilde R(\tau_i)
    -\operatorname{mean}_j\widetilde R(\tau_j)
    }{
    \operatorname{std}_j\widetilde R(\tau_j)+\epsilon
    }$
    for $i=1,\ldots,G$
  \EndFor
  \State Update $\pi_\theta$ using the clipped group-relative
  objective with KL regularization in Section~4.2
\EndFor

\end{algorithmic}
\end{algorithm*}

\paragraph{Training and Inference Distinction.}
Algorithm~\ref{alg:sage_training} summarizes the training-time integration of structural guidance. The two potentials are evaluated only on finite sets of candidate reasoning steps sampled during rollout generation. Their scores affect both candidate selection and group-relative advantages, allowing the policy to internalize structural preferences during optimization. At inference time, the trained policy generates responses directly without additional candidate reweighting, structural-potential computation, or external search.

\subsection{State Encoder and Hyperbolic Embedding}
\label{app:hyperbolic_embedding}

For each intermediate state $s_t$, we first construct a canonical textual representation $x(s_t)$.
In AC task, $x(s_t)$ is the canonicalized group presentation after applying
the generated prefix up to step $t$. In mathematical and free-form reasoning tasks, $x(s_t)$ is
the generated reasoning prefix concatenated with the original prompt.

Let $\mathcal{E}$ denote the frozen reference-model encoder used to featurize intermediate states.
Unless otherwise stated, we use the hidden state from layer $\ell_{\mathrm{enc}}$ at the final generated
token:
\[
h_t
=
\mathcal{E}_{\ell_{\mathrm{enc}}}(x(s_t))
\in \mathbb{R}^{d_h}.
\]
The vector $h_t$ is projected into a lower-dimensional structural representation by a fixed linear
map $W_E\in\mathbb{R}^{d_E\times d_h}$ fitted only on training-rollout states:
\[
z_t = W_E h_t.
\]
In our implementation, $W_E$ is fitted without correctness supervision using training-rollout
representations only. It is fixed before guided rollout generation and is not updated during policy
optimization. We then map $z_t$ into the Poincar\'e ball using radial projection:
\[
E(s_t)
=
\frac{\tanh(\sqrt{c}\|z_t\|_2)}
{\sqrt{c}\|z_t\|_2}
z_t,
\]
where $c>0$ is the curvature parameter. The same encoder and projection are used for the task
anchor $g$:
\[
E(g)
=
\frac{\tanh(\sqrt{c}\|W_E h_g\|_2)}
{\sqrt{c}\|W_E h_g\|_2}
W_E h_g.
\]
For AC task, $g$ is the canonical trivial presentation. For mathematical and free-form
tasks, $g$ is a prompt-conditioned structural anchor derived from the input format and training-rollout
representations, not from test labels, gold rationales, or terminal correctness labels.

The hyperbolic structural potential is
\[
\Psi_H(s_t,a_t,g)
=
\exp\!\left(
-\frac{
d_{\mathbb{D}_c}\!\left(E(s_t\circ a_t),E(g)\right)
}{\kappa}
\right),
\]
where $d_{\mathbb{D}_c}$ is the Poincar\'e distance with curvature $c$, and $\kappa>0$ controls
the sharpness of the guidance signal.

\subsection{Residual Representation and Algebraic Sparsification}
\label{app:residual_probe}

For symbolic domains, the residual $r_t$ is computed from the unresolved
symbolic structure of the current state. In Andrews-Curtis-style tasks,
$r_t$ is derived from the canonicalized presentation after the generated
prefix, including generator counts, relator lengths, and unresolved relator
components. Each candidate move $L_j$ is associated with a structural
subspace $S_j$, and the algebraic sparsification potential is
\[
\Psi_P(r_t,S_j)
=
\frac{\|P_{S_j}r_t\|_2^2}{\|r_t\|_2^2+\epsilon}.
\]

For non-symbolic tasks, we use a finite operator taxonomy over coarse
reasoning-step types. For mathematical reasoning, the operator taxonomy
includes simplification, substitution, numerical evaluation, equation
formation, formula invocation, case split, constraint checking, and
final-answer extraction. For free-form natural reasoning, the taxonomy
includes factual retrieval, comparison, elimination, aggregation, inference,
format normalization, and final-answer commitment.

Given a rollout prefix $s_t$, we compute a frozen encoder representation
\[
h_t=\mathcal{E}_{\ell_{\mathrm{enc}}}(x(s_t)).
\]
The residual vector is predicted by a fixed probe
\[
r_t=W_Rh_t.
\]
The probe is trained only on training-rollout pseudo-labels. Let
$y_t$ denote the rollout-local structural pseudo-label vector, containing
operator type, unresolved constraint coverage, answer-schema status, and
format-consistency indicators. We fit $W_R$ by the regularized regression
objective
\[
W_R
=
\arg\min_W
\sum_{(s_t,y_t)\in\mathcal{D}_{\mathrm{probe}}}
\|Wh_t-y_t\|_2^2
+
\xi\|W\|_F^2 .
\]
No gold answers, gold rationales, ground-truth solution traces, terminal
rewards, test labels, or test-set information are used to train this probe.

For each operator type $j$, we construct a subspace $S_j$ from training
rollouts. Let
\[
\mathcal{R}_j
=
\{r_t:\text{the transition from }s_t\text{ is assigned operator type }j\}
\]
be the residual vectors associated with operator type $j$. We compute the
top $d_j$ principal directions of $\mathcal{R}_j$ and write them as
\[
U_j\in\mathbb{R}^{d_R\times d_j}.
\]
The projection matrix is then
\[
P_{S_j}=U_jU_j^\top .
\]
For a candidate reasoning step $a_t^{(k)}$, we assign a coarse operator
type
\[
j(a_t^{(k)})=C_{\mathrm{op}}(s_t,a_t^{(k)}),
\]
where $C_{\mathrm{op}}$ is the same parser or cluster-based operator
classifier used to construct the pseudo-labels. The algebraic
sparsification score for the candidate step is
\[
\Psi_P(s_t,a_t^{(k)})
=
\frac{
\|P_{S_{j(a_t^{(k)})}}r_t\|_2^2
}{
\|r_t\|_2^2+\epsilon
}.
\]
Thus, $\Psi_P$ measures whether the candidate step's coarse operator type
acts on the unresolved structural residual predicted for the current prefix.
For non-symbolic tasks, this score is a learned training-time structural
prior, not a formal local-admissibility certificate.

\subsection{Meaning Clustering and Fixed-Subset Entropy Filtering}
\label{app:entropy_filtering}

Entropy filtering is used only to define a fixed training subset and is not
used during test-time evaluation. To avoid confounding method performance
with method-dependent prompt selection, we compute the entropy filter once
using rollouts from the same reference policy $\pi_{\mathrm{ref}}$ before
training any compared method.

For each training prompt $q$, we sample $G$ reference rollouts and group the
outputs into meaning clusters $\{c_1,\ldots,c_M\}$. In structured domains,
clustering is implemented by extracting and canonicalizing the final answer
followed by deterministic matching. In free-form domains, we use a binary
semantic-equivalence function
\[
V(q,o_a,o_b)\in\{0,1\}
\]
that returns whether two outputs express the same answer meaning. This
function is applied only to training rollouts for constructing the training
subset.

We compute semantic entropy as
\[
H_{\mathrm{ref}}(q)
=
-\sum_{j=1}^M p(c_j\mid q)\log p(c_j\mid q),
\qquad
p(c_j\mid q)=\frac{|c_j|}{G}.
\]
The fixed filtered training subset is
\[
\mathcal D_{\mathrm{filt}}
=
\{q\in\mathcal D_{\mathrm{train}}:
\delta_{\mathrm{low}}<H_{\mathrm{ref}}(q)<\delta_{\mathrm{high}}\}.
\]
All compared methods in the controlled ablation study, including GRPO,
EMPO, PRM-GRPO, and \model{}, are trained on the same
$\mathcal D_{\mathrm{filt}}$ with the same number of prompts, rollout groups,
optimization steps, decoding constraints, and KL coefficient. Therefore,
differences in performance cannot be attributed to method-specific prompt
retention.

All reported evaluation results are computed on the full benchmark test
sets without filtering test examples, without semantic clustering, and
without evaluating $\Psi_P$ or $\Psi_H$ at inference time.

\textbf{Training Data.}
We use separate training pools for mathematical and natural reasoning, with all 11 benchmark test sets reserved exclusively for evaluation. For mathematical reasoning, we randomly sample 20,000 prompts from NuminaMath-CoT, evaluating the post-trained models on MATH, Minerva Math, OlympiadBench, AIME24, AMC23, GSM8K, and Putnam. For natural reasoning, we use 18,000 prompts from Natural Reasoning after filtering examples whose prompt, reference answer, or reference-model response exceeds 4,096 tokens, evaluating on MMLU-Pro, GPQA, BBH-H, and ARC-C. Following entropy-based prompt filtering, all methods within each controlled comparison use the same fixed training subset and optimization settings (Appendix~E.5). The 1,190 independently constructed AC presentations are used exclusively in the separate symbolic AC setting and are not mixed with either LLM training pool.

\paragraph{EMPO Baseline.}
EMPO is a fully unsupervised reinforcement learning baseline that groups sampled responses according to semantic equivalence of their final answers and constructs entropy-based rewards from the empirical distribution over answer clusters. By encouraging low-entropy, self-consistent predictions, EMPO provides a learning signal without gold answers, verified reasoning trajectories, or external reward models. It additionally employs entropy thresholds and KL regularization for training stability. Unlike SAGE, which provides prefix-level guidance through algebraic sparsification and hyperbolic structural potentials, EMPO operates primarily on the final-response distribution and does not explicitly model the structural feasibility or depth-wise progress of intermediate reasoning trajectories.

\paragraph{AC Path Solving.}
We distinguish three evaluation metrics for the Andrews--Curtis (AC) task. ACValidity measures the percentage of syntactically and logically valid intermediate transitions, while Lean-Verified Proofs measures end-to-end success certified by the Lean proof assistant. These two metrics serve as the primary evaluation criteria in Section~5.3.

ACPathSolving is an additional path-level diagnostic reported in the component ablation and reward-side analyses. Following the long-horizon AC formulation, it measures the percentage of evaluated instances for which a generated trajectory identifies a valid sequence of Andrews--Curtis moves reaching the target trivial presentation within the allowed horizon $T$:
\begin{equation}
\mathrm{ACPathSolving}
=
\frac{100}{N_{\mathrm{AC}}}
\sum_{i=1}^{N_{\mathrm{AC}}}
\mathbf{1}\left[
\exists\, t \leq T:
s_t^{(i)} = s_{\mathrm{triv}}^{(i)}
\land
\forall k<t,\,
\operatorname{Adm}_{\mathcal S}
(s_k^{(i)},a_k^{(i)},s_{k+1}^{(i)})=1
\right],
\label{eq:ac_path_solving}
\end{equation}
where $N_{\mathrm{AC}}$ denotes the number of evaluated instances and $s_{\mathrm{triv}}^{(i)}$ is the target trivial presentation. Unlike ACValidity, which evaluates local transition validity, ACPathSolving evaluates whether a valid trajectory reaches the target within the prescribed horizon. Lean-Verified Proofs further requires independent formal verification.

\section{Hyperparameter Ablation}
\label{app:hyperparameter_ablation}

We evaluate the robustness of \model{} on the Olympiad dataset under three sources of variation:
Pass@$K$ scaling, decoding temperature, and the KL regularization coefficient $\beta_{\mathrm{KL}}$.
As shown in Figure~\ref{fig:ablation}, \model{} maintains a consistent performance lead over GRPO across sampling budgets.
At Pass@64, \model{} reaches approximately $87\%$, compared with approximately $84\%$ for GRPO.

Lower decoding temperatures generally favor more deterministic reasoning.
However, \model{} remains competitive under higher stochasticity, indicating that structural post-training improves the quality of the sampled reasoning distribution rather than merely exploiting a narrow decoding regime.
We also observe that \model{} is less sensitive to the KL coefficient than GRPO.
While GRPO performance varies substantially across $\beta_{\mathrm{KL}}$, \model{} with $\beta_{\mathrm{KL}}=10^{-3}$ outperforms the strongest GRPO variants in this sweep.

\begin{figure*}[h!]
    \centering
    \includegraphics[width=\textwidth]{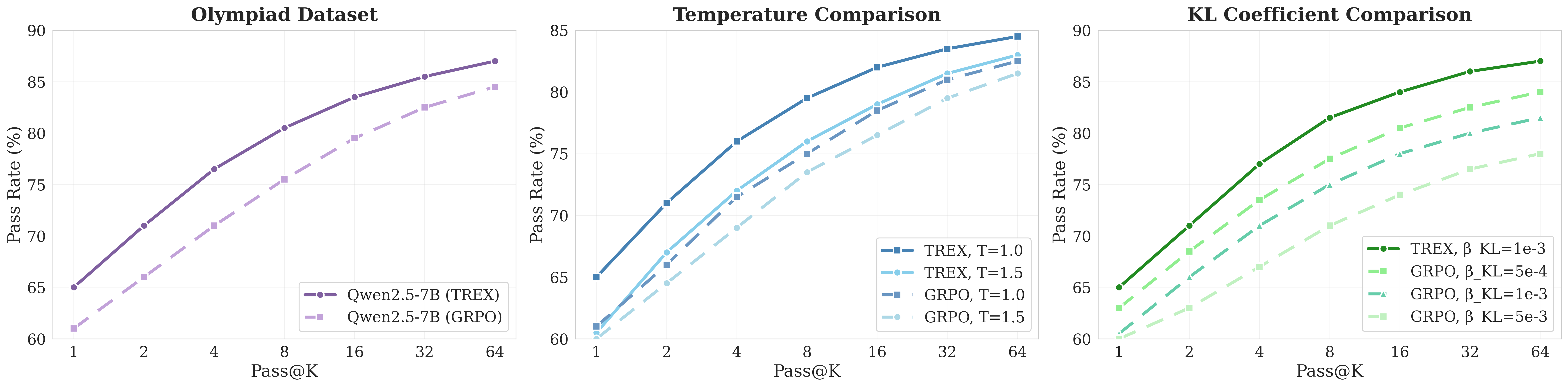}
    \caption{
    Hyperparameter ablation on the Olympiad dataset.
    Left: \model{} scales consistently across Pass@$K$ budgets.
    Middle: \model{} remains robust under different decoding temperatures.
    Right: \model{} is less sensitive to the KL coefficient $\beta_{\mathrm{KL}}$ than GRPO.
    }
    \label{fig:ablation}
    \vspace{-1em}
\end{figure*}

\section{Process Reward Model Baseline}
\label{app:prm_baseline}

To further compare \model{} against dense process-level reward shaping, we include a
process-reward-model baseline, denoted PRM-GRPO. This baseline addresses whether the
improvements of \model{} can be explained simply by adding a learned dense reward signal during
post-training.

\paragraph{PRM construction.}
PRM-GRPO trains a step-level reward model $R_{\phi}(s_t,a_t)$ on the same training rollouts used
by the RL baselines. For mathematical and free-form reasoning tasks, process labels are derived
from terminal-outcome-labeled rollouts through outcome-to-prefix credit assignment: prefixes from
successful rollouts are treated as positive process examples, while prefixes from failed rollouts are
treated as negative process examples. For AC task, we additionally use the
task interface to derive local symbolic-validity targets for generated transitions. The PRM is trained
only on training rollouts and does not use test-set labels or test-time feedback.

\paragraph{PRM-guided post-training.}
During post-training, PRM-GRPO uses the same rollout budget, decoding constraints, KL coefficient,
and optimization schedule as GRPO and \model{}. The only difference from GRPO is that the sparse
terminal reward is augmented with the learned process reward:
\[
\widetilde{R}_{\mathrm{PRM}}(\tau)
=
R_{\mathrm{term}}(\tau)
+
\eta_{\mathrm{PRM}}
\frac{1}{T}
\sum_{t=1}^{T}
R_{\phi}(s_t,a_t).
\]
The resulting trajectory-level reward is then normalized within each rollout group and optimized
with the same group-relative policy update. The PRM is used only during post-training reward
shaping and is not used to filter test examples or guide inference.

\section{Additional Results}
\label{app:addition}

\subsection{Generalization Beyond Symbolic Reasoning}
\label{app:generalization}

\paragraph{Extension to Non-Tree Reasoning Spaces.}
SAGE does not require the underlying environment state graph to be a tree. For non-symbolic tasks, structural guidance operates on history-conditioned reasoning prefixes, where each state representation combines the original task input with the generated reasoning or action history. Consequently, trajectories reaching the same environment state through different histories can still receive distinct structural scores. This formulation allows SAGE to operate on reasoning problems involving cyclic dependencies, multi-hop evidence integration, and non-tree action spaces.

For MuSiQue, the residual representation captures unresolved question requirements, including entities, relations, and evidence dependencies, while the target anchor encodes prompt-conditioned structural information estimated from training rollouts. For ALFWorld, the residual captures unresolved task requirements and subgoals conditioned on the instruction, current observation, and action history. Both settings use the learned structural-prior construction described in Appendix~E.

\paragraph{Multi-Hop Question Answering.}
Table~\ref{tab:musique} presents results on MuSiQue across three model scales. SAGE consistently outperforms both GRPO and EMPO in exact-match accuracy and token-level Answer F1. Compared with GRPO, SAGE improves Answer EM by 5.1, 4.6, and 4.6 percentage points on the 9B, 27B, and 35B models, respectively. The corresponding Answer F1 gains are 4.8, 4.2, and 4.6 points. These results suggest that structural guidance remains beneficial when successful reasoning requires integrating evidence across multiple dependent hops rather than following a predefined symbolic transformation sequence.

\paragraph{Agentic Reasoning.}
Table~\ref{tab:alfworld} reports task success and action efficiency on ALFWorld. SAGE improves task success over GRPO by 7.5, 7.1, and 7.1 percentage points across the three model scales, while reducing the average number of environment actions among successful episodes by 3.3, 2.9, and 2.7 steps, respectively. The simultaneous improvements in success rate and action efficiency suggest that structural guidance helps policies maintain task-directed progress across extended interaction sequences.

Together, these results provide empirical evidence that SAGE can generalize beyond explicit symbolic reasoning to multi-hop question answering and interactive agentic tasks. However, the learned structural priors in these settings remain approximate, and the exact symbolic guarantees of SCA do not directly extend to arbitrary non-symbolic environments.

\begin{table}[h!]
\centering
\setlength{\tabcolsep}{3pt}
\caption{
Accuracy (\%) on free-form natural reasoning benchmarks over 5 random seeds, reported as mean with {\tiny standard deviation}.
The best mean is in \textbf{bold} with second best in \underline{underline}.
}
\label{tab:natural_reasoning_seed}
\resizebox{\textwidth}{!}{
\begin{tabular}{l c c c c c c c c}
\toprule
\textbf{Model} & \textbf{STEM}
& \multicolumn{4}{c}{\textbf{MMLU-Pro}} 
& \textbf{GPQA} 
& \textbf{BBH-H} 
& \textbf{ARC-C} \\
\cmidrule(lr){3-6}
 &  &  \footnotesize{Humanity} & \footnotesize{Social} & Other & Avg. &  &  &  \\
\midrule
\multicolumn{9}{l}{\emph{9B models}} \\
Qwen3.5            
& \meanstd{12.71}{0.21} & \meanstd{8.02}{0.18}  & \meanstd{14.95}{0.25} & \meanstd{10.21}{0.20} & \meanstd{11.06}{0.16} & \meanstd{10.91}{0.31} & \meanstd{21.58}{0.38} & \meanstd{18.49}{0.35} \\
Qwen3.5 w/SFT      
& \meanstd{20.18}{0.34} & \meanstd{11.36}{0.27} & \meanstd{28.97}{0.43} & \meanstd{19.22}{0.36} & \meanstd{19.85}{0.31} & \meanstd{12.31}{0.42} & \meanstd{32.44}{0.55} & \meanstd{29.56}{0.51} \\
Qwen3.5 w/GRPO     
& \underline{\meanstd{33.38}{0.56}} & \underline{\meanstd{28.31}{0.49}} & \underline{\meanstd{50.41}{0.62}} & \underline{\meanstd{39.26}{0.55}} & \underline{\meanstd{39.33}{0.47}} & \meanstd{18.21}{0.73} & \meanstd{41.69}{0.84} & \meanstd{37.82}{0.79} \\
Qwen3.5 w/EMPO     
& \meanstd{32.57}{0.53} & \meanstd{27.32}{0.51} & \meanstd{48.73}{0.66} & \meanstd{37.69}{0.58} & \meanstd{37.91}{0.50} & \textbf{\meanstd{21.11}{0.69}} & \underline{\meanstd{44.04}{0.82}} & \underline{\meanstd{39.73}{0.76}} \\
\textbf{Qwen3.5 w/\model} 
& \textbf{\meanstd{34.11}{0.48}} & \textbf{\meanstd{29.08}{0.44}} & \textbf{\meanstd{51.17}{0.59}} & \textbf{\meanstd{39.71}{0.52}} & \textbf{\meanstd{39.99}{0.43}} & \underline{\meanstd{20.86}{0.64}} & \textbf{\meanstd{45.31}{0.71}} & \textbf{\meanstd{42.04}{0.68}} \\
\midrule
\multicolumn{9}{l}{\emph{27B models}} \\
Qwen3.6            
& \meanstd{30.29}{0.25} & \meanstd{24.34}{0.22} & \meanstd{46.55}{0.34} & \meanstd{35.31}{0.29} & \meanstd{35.40}{0.24} & \meanstd{16.09}{0.38} & \meanstd{38.70}{0.49} & \meanstd{34.78}{0.46} \\
Qwen3.6 w/SFT      
& \meanstd{34.18}{0.31} & \meanstd{28.51}{0.27} & \meanstd{41.52}{0.39} & \meanstd{37.28}{0.33} & \meanstd{35.77}{0.29} & \meanstd{22.67}{0.47} & \meanstd{45.27}{0.57} & \meanstd{41.12}{0.53} \\
Qwen3.6 w/GRPO     
& \textbf{\meanstd{57.74}{0.49}} & \underline{\meanstd{37.0}{0.43}} & \textbf{\meanstd{65.16}{0.54}} & \underline{\meanstd{57.55}{0.48}} & \underline{\meanstd{53.24}{0.41}} & \textbf{\meanstd{34.29}{0.62}} & \meanstd{55.74}{0.71} & \meanstd{51.78}{0.67} \\
Qwen3.6 w/EMPO     
& \meanstd{53.96}{0.51} & \meanstd{35.58}{0.45} & \meanstd{60.04}{0.58} & \meanstd{52.18}{0.50} & \meanstd{49.27}{0.44} & \meanstd{29.52}{0.66} & \underline{\meanstd{57.19}{0.69}} & \underline{\meanstd{53.26}{0.63}} \\
\textbf{Qwen3.6 w/\model} 
& \underline{\meanstd{56.31}{0.44}} & \textbf{\meanstd{37.91}{0.39}} & \underline{\meanstd{64.43}{0.51}} & \textbf{\meanstd{58.25}{0.43}} & \textbf{\meanstd{53.53}{0.37}} & \underline{\meanstd{32.51}{0.57}} & \textbf{\meanstd{60.25}{0.62}} & \textbf{\meanstd{57.11}{0.58}} \\
\midrule
\multicolumn{9}{l}{\emph{35B models}} \\
Qwen3.5              
& \meanstd{45.08}{0.18} & \meanstd{36.31}{0.17} & \meanstd{52.29}{0.25} & \meanstd{44.28}{0.21} & \meanstd{44.29}{0.18} & \meanstd{31.05}{0.34} & \meanstd{47.52}{0.41} & \meanstd{45.37}{0.39} \\
Qwen3.5 w/SFT        
& \meanstd{49.84}{0.23} & \meanstd{38.26}{0.20} & \meanstd{54.47}{0.31} & \meanstd{48.62}{0.26} & \meanstd{47.12}{0.22} & \meanstd{28.97}{0.39} & \meanstd{52.96}{0.46} & \meanstd{49.18}{0.43} \\
Qwen3.5 w/GRPO       
& \underline{\meanstd{63.58}{0.38}} & \underline{\meanstd{43.19}{0.34}} & \underline{\meanstd{69.08}{0.45}} & \underline{\meanstd{60.71}{0.39}} & \underline{\meanstd{57.66}{0.33}} & \underline{\meanstd{35.78}{0.51}} & \meanstd{63.29}{0.57} & \meanstd{60.91}{0.55} \\
Qwen3.5 w/EMPO       
& \meanstd{61.96}{0.41} & \meanstd{42.02}{0.36} & \meanstd{68.61}{0.47} & \meanstd{59.54}{0.42} & \meanstd{56.72}{0.35} & \meanstd{35.43}{0.54} & \underline{\meanstd{65.02}{0.55}} & \underline{\meanstd{62.97}{0.51}} \\
\textbf{Qwen3.5 w/\model} 
& \textbf{\meanstd{65.27}{0.33}} & \textbf{\meanstd{44.51}{0.30}} & \textbf{\meanstd{71.22}{0.40}} & \textbf{\meanstd{62.25}{0.35}} & \textbf{\meanstd{59.33}{0.29}} & \textbf{\meanstd{38.58}{0.46}} & \textbf{\meanstd{69.07}{0.49}} & \textbf{\meanstd{66.41}{0.47}} \\
\bottomrule
\end{tabular}
}
\end{table}

\begin{table}[h!]
\centering
\caption{Comparison on Math Reasoning Benchmarks.}
\label{tab:math_reasoning}
\resizebox{\textwidth}{!}{%
\begin{tabular}{llcccccccc}
\toprule
\textbf{Backbone} & \textbf{Baseline} &
\textbf{MATH} & \textbf{Minerva} &
\textbf{Olympiad} & \textbf{AIME24} &
\textbf{AMC23} & \textbf{GSM8K} &
\textbf{Putnam} & \textbf{Avg.} \\
\midrule
\multirow{5}{*}{DeepSeekMath-7B}
&        & 67.20 & 34.50 & 39.80 & 25.70 & 50.30 & 78.40 & 10.80 & 43.81 \\
& w/SFT  & 72.80 & 39.40 & 46.10 & 31.80 & 56.70 & 83.20 & 13.60 & 49.09 \\
& w/GRPO & 79.30 & 45.70 & 55.80 & 45.20 & 63.90 & 88.10 & 17.40 & 56.49 \\
& w/EMPO & 78.80 & 46.10 & \textbf{59.00} & \textbf{49.20} & 64.80 & 88.90 & 17.90 & 57.81 \\
& w/SAGE & \textbf{81.40} & \textbf{48.60} & 58.70 & 48.70 & \textbf{66.70} & \textbf{90.00} & \textbf{19.80} & \textbf{59.13} \\
\midrule

\multirow{5}{*}{DeepSeek-Prover-V2-7B}
&        & 70.40 & 37.80 & 49.60 & 34.20 & 54.80 & 80.10 & 14.70 & 48.80 \\
& w/SFT  & 74.90 & 41.60 & 53.80 & 38.50 & 59.10 & 83.70 & 16.30 & 52.56 \\
& w/GRPO & 80.80 & 47.90 & 61.40 & 50.60 & 65.20 & \textbf{91.00} & 20.40 & 59.61 \\
& w/EMPO & 80.20 & 48.40 & 62.10 & \textbf{54.50} & 65.80 & 89.10 & 20.80 & 60.13 \\
& w/SAGE & \textbf{82.60} & \textbf{50.70} & \textbf{64.90} & 53.80 & \textbf{67.40} & 90.70 & \textbf{22.60} & \textbf{61.81} \\
\midrule

\multirow{5}{*}{Kimina-Prover-7B}
&        & 71.60 & 38.90 & 51.20 & 35.80 & 56.10 & 80.80 & 15.60 & 50.00 \\
& w/SFT  & 75.80 & 42.70 & 55.60 & 40.40 & 60.30 & 84.10 & 17.10 & 53.71 \\
& w/GRPO & \textbf{83.40} & 48.60 & 62.90 & 52.10 & 66.40 & 89.00 & 21.10 & 60.50 \\
& w/EMPO & 80.70 & 49.10 & 63.40 & 54.80 & 67.00 & 89.40 & 21.60 & 60.86 \\
& w/SAGE & 83.10 & \textbf{51.30} & \textbf{66.10} & \textbf{55.20} & \textbf{68.60} & \textbf{90.70} & \textbf{23.40} & \textbf{62.63} \\
\midrule

\multirow{5}{*}{Kimina-Prover-Distill-8B}
&        & 73.80 & 41.20 & 54.60 & 39.30 & 59.20 & 82.40 & 17.30 & 52.54 \\
& w/SFT  & 77.40 & 45.10 & 58.70 & 43.70 & 63.10 & 85.60 & 18.90 & 56.07 \\
& w/GRPO & 82.60 & 50.50 & 65.10 & 55.30 & 68.20 & \textbf{92.00} & 22.80 & 62.36 \\
& w/EMPO & 82.10 & \textbf{53.40} & 65.80 & \textbf{59.00} & 68.90 & 90.50 & 23.20 & 63.27 \\
& w/SAGE & \textbf{84.20} & 53.20 & \textbf{68.00} & 58.10 & \textbf{70.40} & 91.70 & \textbf{25.00} & \textbf{64.37} \\
\midrule

\multirow{5}{*}{Llama-3.3-70B-Instruct}
&        & 66.08 & 33.61 & 32.94 & 17.31 & 29.02 & 80.47 & 9.91 & 38.48 \\
& w/SFT  & 72.40 & 40.20 & 44.60 & 30.80 & 48.90 & 85.30 & 14.10 & 48.04 \\
& w/GRPO & 80.10 & 47.80 & 59.40 & 50.60 & 64.70 & 90.50 & 19.00 & 58.87 \\
& w/EMPO & 79.60 & \textbf{51.30} & 60.10 & \textbf{55.00} & 65.40 & 91.00 & 19.60 & 60.29 \\
& w/SAGE & \textbf{82.50} & 51.00 & \textbf{63.40} & 54.00 & \textbf{67.10} & \textbf{92.30} & \textbf{21.70} & \textbf{61.71} \\

\bottomrule
\end{tabular}%
}
\end{table}

\subsection{Mechanism Validation Under Sparse Rewards}
\label{app:mechanism_validation}

\paragraph{Non-Vanishing Structural Advantages.}
Table~\ref{tab:ac_reward_diagnostics} evaluates whether SAGE maintains informative optimization signals when terminal rewards are sparse. Among rollout groups in which all terminal rewards are zero, SAGE retains non-zero within-group advantages in 56.8\% of cases, compared with 0.0\% for GRPO and 14.2\% for EMPO. This observation is consistent with Proposition~4.1, which establishes that variation in structural potentials can produce non-vanishing group-relative advantages even when terminal feedback is uninformative.

Importantly, stronger optimization signals are accompanied by improvements in independent correctness-related metrics. Relative to GRPO, SAGE increases Lean-verified intermediate steps from 18.6 to 31.9 per 1,000 tokens, a 71\% improvement, while improving AC Path Solving from 23.05\% to 31.76\% and Lean-Verified success from 14.64\% to 23.69\%. These findings support the interpretation that structural guidance provides useful learning signals in sparse-reward regimes rather than merely increasing auxiliary reward values.

\paragraph{Representative Trajectory Analysis.}
Table~\ref{tab:ac_case_study} provides a trajectory-level illustration using the median-length SAGE trajectory among instances successfully solved by SAGE. All methods are evaluated on the same AC instance, with the first 12 transitions compared under a matched horizon.

SAGE maintains 12 valid transitions out of 12 and reduces residual complexity from 18 to 5. In comparison, GRPO achieves 6 valid transitions and increases residual complexity from 18 to 19, while EMPO achieves 7 valid transitions and reduces residual complexity only to 16. SAGE subsequently completes the AC path and passes independent Lean verification, whereas neither baseline succeeds on this instance.

This case illustrates how maintaining a valid prefix and reducing residual complexity can jointly support long-horizon completion. As a selected representative example, it complements rather than replaces the aggregate evaluation.

\paragraph{Structural Reward and Terminal Correctness.}
Structural reward density and terminal correctness measure different aspects of reasoning quality. Higher structural rewards do not necessarily imply greater end-to-end accuracy, particularly in heterogeneous free-form reasoning settings such as BBH-H, where locally plausible reasoning steps may still lead to incorrect final answers.

To examine this distinction, Table~\ref{tab:reward_correctness_alignment} evaluates the correspondence between individual structural potentials and terminal correctness on held-out BBH-H rollouts. Algebraic sparsification achieves a Spearman correlation of 0.17 and success-ranking AUROC of 0.58, while hyperbolic guidance achieves 0.20 and 0.60, respectively. In contrast, the combined SAGE signal achieves substantially stronger alignment, with $\rho=0.37$ and AUROC $=0.69$.

These results suggest that the two potentials capture complementary aspects of reasoning behavior: algebraic sparsification measures structural exploration compatibility, whereas hyperbolic guidance captures sustained depth-wise progress. Their combination provides a more informative structural signal than either component alone, helping explain why higher reward density from an isolated component does not necessarily translate into improved terminal accuracy.

Overall, the AC diagnostics, trajectory-level case study, and BBH-H correctness-alignment analysis provide complementary evidence for SAGE's mechanism: structural guidance sustains optimization signals under sparse rewards, promotes long-horizon progress, and exhibits stronger correspondence with correctness when both components are combined. These findings are consistent with the proposed mechanism, without establishing that structural rewards guarantee correctness for every trajectory.

\begin{table}[h!]
\centering
\caption{Comparison on Natural Reasoning Benchmarks.}
\label{tab:natural_reasoning_extend}
\resizebox{\textwidth}{!}{%
\begin{tabular}{llcccccccc}
\toprule
\multirow{2}{*}{\textbf{Backbone}} &
\multirow{2}{*}{\textbf{Baseline}} &
\multicolumn{5}{c}{\textbf{MMLU-Pro}} &
\multirow{2}{*}{\textbf{GPQA}} &
\multirow{2}{*}{\textbf{BBH-H}} &
\multirow{2}{*}{\textbf{ARC-C}} \\
\cmidrule(lr){3-7}
& & \textbf{STEM} & \textbf{Humanity} &
\textbf{Social} & \textbf{Other} & \textbf{Avg.} &
& & \\
\midrule

\multirow{5}{*}{DeepSeekMath-7B}
&        & 26.40 & 21.80 & 35.70 & 29.20 & 28.27 & 19.60 & 31.80 & 30.70 \\
& w/SFT  & 31.20 & 24.60 & 39.40 & 33.10 & 32.08 & 21.80 & 35.20 & 34.40 \\
& w/GRPO & \textbf{41.70} & 31.40 & 50.20 & 42.10 & 41.35 & 26.70 & 44.80 & 42.60 \\
& w/EMPO & 38.90 & 32.10 & 49.50 & 41.70 & 40.55 & \textbf{27.40} & 46.00 & 43.10 \\
& w/SAGE & 41.50 & \textbf{33.70} & \textbf{52.10} & \textbf{44.00} & \textbf{42.83} & 27.00 & \textbf{47.90} & \textbf{45.30} \\
\midrule

\multirow{5}{*}{DeepSeek-Prover-V2-7B}
&        & 23.80 & 19.70 & 32.40 & 26.10 & 25.50 & 18.40 & 29.10 & 28.50 \\
& w/SFT  & 28.60 & 22.30 & 36.80 & 30.20 & 29.48 & 20.60 & 33.00 & 31.70 \\
& w/GRPO & 36.70 & 28.80 & 46.30 & 38.50 & 37.58 & \textbf{25.80} & 41.20 & 39.40 \\
& w/EMPO & 35.90 & \textbf{31.40} & 45.80 & 38.10 & 37.80 & 25.60 & 42.00 & 40.20 \\
& w/SAGE & \textbf{38.40} & 31.00 & \textbf{48.60} & \textbf{40.70} & \textbf{39.67} & 25.20 & \textbf{44.30} & \textbf{42.10} \\
\midrule

\multirow{5}{*}{Kimina-Prover-7B}
&        & 25.10 & 20.80 & 34.10 & 27.30 & 26.83 & 18.90 & 30.40 & 29.60 \\
& w/SFT  & 29.70 & 23.50 & 37.90 & 31.50 & 30.65 & 21.10 & 34.20 & 32.80 \\
& w/GRPO & 37.80 & 29.90 & \textbf{50.10} & 39.80 & 39.40 & 25.70 & 42.70 & 40.60 \\
& w/EMPO & 37.10 & 30.50 & 47.00 & 39.40 & 38.50 & 26.00 & \textbf{46.10} & 41.10 \\
& w/SAGE & \textbf{39.60} & \textbf{32.20} & 49.80 & \textbf{42.00} & \textbf{40.90} & \textbf{26.50} & 45.80 & \textbf{43.40} \\
\midrule

\multirow{5}{*}{Kimina-Prover-Distill-8B}
&        & 28.40 & 23.10 & 37.20 & 30.60 & 29.83 & 20.80 & 33.60 & 32.20 \\
& w/SFT  & 32.70 & 25.90 & 41.10 & 34.40 & 33.52 & 22.90 & 36.90 & 35.10 \\
& w/GRPO & 40.50 & 32.40 & 50.60 & 42.80 & 41.58 & 27.90 & \textbf{48.70} & 43.20 \\
& w/EMPO & 39.80 & \textbf{35.00} & 50.10 & 42.40 & 41.83 & \textbf{28.50} & 46.10 & 43.80 \\
& w/SAGE & \textbf{42.30} & 34.60 & \textbf{52.90} & \textbf{44.90} & \textbf{43.67} & 28.10 & 48.40 & \textbf{46.00} \\
\midrule

\multirow{5}{*}{Llama-3.3-70B-Instruct}
&        & 52.80 & 43.60 & 61.90 & 54.10 & 53.10 & 38.20 & 59.40 & 57.80 \\
& w/SFT  & 55.10 & 45.30 & 64.20 & 56.60 & 55.30 & 39.70 & 61.80 & 59.90 \\
& w/GRPO & 61.80 & 50.40 & \textbf{72.60} & 63.20 & 62.00 & 43.60 & 66.70 & 64.10 \\
& w/EMPO & 61.00 & 51.10 & 69.80 & 62.70 & 61.15 & \textbf{44.20} & 68.10 & 64.80 \\
& w/SAGE & \textbf{63.40} & \textbf{52.80} & 72.10 & \textbf{65.00} & \textbf{63.32} & 43.90 & \textbf{69.80} & \textbf{67.10} \\

\bottomrule
\end{tabular}%
}
\end{table}

\begin{table}[h!]
\centering
\caption{Multi-hop QA results on MuSiQue. We report exact-match accuracy (EM) and token-level Answer F1.}
\label{tab:musique}
\resizebox{0.7\textwidth}{!}{%
\begin{tabular}{llcc}
\toprule
\textbf{Model} & \textbf{Method} &
\textbf{Answer EM (\%) $\uparrow$} &
\textbf{Answer F1 (\%) $\uparrow$} \\
\midrule

\multirow{3}{*}{Qwen3.5-9B}
& GRPO & 49.8 & 62.4 \\
& EMPO & 50.6 & 63.1 \\
& SAGE & \textbf{54.9} & \textbf{67.2} \\
\midrule

\multirow{3}{*}{Qwen3.6-27B}
& GRPO & 55.7 & 67.9 \\
& EMPO & 56.4 & 68.5 \\
& SAGE & \textbf{60.3} & \textbf{72.1} \\
\midrule

\multirow{3}{*}{Qwen3.5-35B}
& GRPO & 59.1 & 70.8 \\
& EMPO & 59.8 & 71.5 \\
& SAGE & \textbf{63.7} & \textbf{75.4} \\

\bottomrule
\end{tabular}%
}
\end{table}

\begin{table}[h!]
\centering
\caption{Agentic reasoning results on ALFWorld. Task Success measures completed goals, while Avg. Steps reports the number of environment actions among successful episodes.}
\label{tab:alfworld}
\resizebox{0.7\textwidth}{!}{%
\begin{tabular}{llcc}
\toprule
\textbf{Model} & \textbf{Method} &
\textbf{Task Success (\%) $\uparrow$} &
\textbf{Avg. Steps (Success) $\downarrow$} \\
\midrule

\multirow{3}{*}{Qwen3.5-9B}
& GRPO & 57.2 & 19.6 \\
& EMPO & 58.4 & 18.9 \\
& SAGE & \textbf{64.7} & \textbf{16.3} \\
\midrule

\multirow{3}{*}{Qwen3.6-27B}
& GRPO & 63.5 & 17.8 \\
& EMPO & 64.1 & 17.2 \\
& SAGE & \textbf{70.6} & \textbf{14.9} \\
\midrule

\multirow{3}{*}{Qwen3.5-35B}
& GRPO & 67.8 & 16.5 \\
& EMPO & 68.4 & 16.0 \\
& SAGE & \textbf{74.9} & \textbf{13.8} \\

\bottomrule
\end{tabular}%
}
\end{table}

\begin{table}[h!]
\centering
\caption{Training cost on Qwen3.5-35B. Reported costs include auxiliary reward acquisition and one-time preprocessing.}
\label{tab:training_cost}
\resizebox{\columnwidth}{!}{%
\begin{tabular}{lcccc}
\toprule
\textbf{Method} &
\textbf{Train Wall-clock $\downarrow$} &
\textbf{GPU-hours $\downarrow$} &
\textbf{Peak GPU Memory $\downarrow$} &
\textbf{Lean-Verified $\uparrow$} \\
\midrule
GRPO     & 1.00$\times$ & 1.00$\times$ & 1.00$\times$ & 14.64 \\
GRPO-PRM & 1.30$\times$ & 1.28$\times$ & 1.15$\times$ & 17.36 \\
\textbf{SAGE} &
\textbf{1.16$\times$} &
\textbf{1.15$\times$} &
\textbf{1.08$\times$} &
\textbf{23.69} \\
\bottomrule
\end{tabular}%
}
\end{table}

\begin{table}[h!]
\centering
\caption{Reward-side diagnostics on AC. Zero-Terminal Groups w/ Non-zero Advantage denotes rollout groups where all terminal rewards are zero but the within-group advantage remains non-zero.}
\label{tab:ac_reward_diagnostics}
\resizebox{\textwidth}{!}{%
\begin{tabular}{lccccc}
\toprule
\textbf{Method} &
\textbf{Terminal Rew./1k $\uparrow$} &
\textbf{Zero-Terminal Groups w/ Non-zero Advantage (\%) $\uparrow$} &
\textbf{Lean-Verified Mid-Steps/1k $\uparrow$} &
\textbf{AC Path Solving $\uparrow$} &
\textbf{Lean-Verified $\uparrow$} \\
\midrule
GRPO & 1.8 & 0.0 & 18.6 & 23.05 & 14.64 \\
EMPO & 1.6 & 14.2 & 17.9 & 21.52 & 13.78 \\
\textbf{SAGE} &
\textbf{3.7} &
\textbf{56.8} &
\textbf{31.9} &
\textbf{31.76} &
\textbf{23.69} \\
\bottomrule
\end{tabular}%
}
\end{table}

\begin{table}[h!]
\centering
\caption{Representative AC case study. We compare all methods on the same instance over the first 12 transitions, using the median-length SAGE trajectory among instances solved by SAGE.}
\label{tab:ac_case_study}
\resizebox{0.9\textwidth}{!}{%
\begin{tabular}{lccccc}
\toprule
\textbf{Method} &
\textbf{Lean-Valid Steps $\uparrow$} &
\textbf{Longest Valid Prefix $\uparrow$} &
\textbf{Residual Complexity $\downarrow$} &
\textbf{Final AC Path Solved} &
\textbf{Final Lean-Verified} \\
\midrule
GRPO & 6/12 & 6 & $18 \rightarrow 19$ & $\times$ & $\times$ \\
EMPO & 7/12 & 7 & $18 \rightarrow 16$ & $\times$ & $\times$ \\
\textbf{SAGE} &
\textbf{12/12} &
\textbf{12} &
$\boldsymbol{18 \rightarrow 5}$ &
$\checkmark$ &
$\checkmark$ \\
\bottomrule
\end{tabular}%
}
\end{table}

\begin{table}[h!]
\centering
\caption{Reward--correctness alignment on held-out BBH-H rollouts. We evaluate how individual structural potentials and their combined SAGE signal correlate with terminal correctness.}
\label{tab:reward_correctness_alignment}
\resizebox{\columnwidth}{!}{%
\begin{tabular}{lcc}
\toprule
\textbf{Structural Signal} &
\textbf{Spearman $\rho$ $\uparrow$} &
\textbf{Success-ranking AUROC $\uparrow$} \\
\midrule
Algebraic Sparsification ($\Psi_P$)
& 0.17 & 0.58 \\
Hyperbolic Guidance ($\Psi_H$)
& 0.20 & 0.60 \\
\textbf{Full SAGE} ($\Psi_{\mathrm{SAGE}}$)
& \textbf{0.37} & \textbf{0.69} \\
\bottomrule
\end{tabular}%
}
\end{table}

\section{Computation Resources}
\label{sec:compute}

All experiments were conducted on an internal GPU cluster with 4 NVIDIA A100 GPUs and 2 NVIDIA H200 GPUs. We did not train any foundation model from scratch; compute was mainly used for post-training rollout generation, structural-potential computation, policy optimization, ablations, and benchmark evaluation. SAGE adds training-time overhead for computing algebraic sparsification and hyperbolic structural guidance, but incurs no extra inference-time cost because the guidance is absorbed into the trained policy.

\section{Limitations}
\label{sec:limit}
SAGE relies on training-time structural priors whose quality depends on the task interface. In explicit symbolic domains, local admissibility can be defined exactly, whereas mathematical and free-form reasoning rely on approximate learned residuals, anchors, and operator subspaces. Consequently, the theoretical concentration guarantees apply directly to symbolic settings and conditionally to settings where the learned potentials separate productive and unproductive prefixes.

Moreover, SAGE's structural reward augmentation is not strict potential-difference shaping and therefore does not guarantee policy invariance. Although length averaging mitigates incentives for artificially extending trajectories, imperfect structural priors may introduce reward misspecification or favor structurally plausible but incorrect reasoning paths. KL regularization and clipped updates constrain policy drift but do not eliminate this risk. We therefore evaluate SAGE using independent correctness metrics, including terminal task success and formal Lean verification, rather than relying solely on structural rewards.

Finally, SAGE incurs additional training-time overhead from candidate-step scoring, residual probing, and hyperbolic distance computation, although inference requires no additional structural scoring. Future work should investigate more robust structural priors, stronger correctness alignment, and tighter guarantees for non-symbolic reasoning.

\section{Broader Impact}
\label{sec:broader_impact}

This work aims to improve long-horizon reasoning under sparse-reward regimes. Its potential positive impact is to make LLM reasoning more reliable in domains requiring extended symbolic or semi-symbolic reasoning, such as mathematics, formal verification, and scientific reasoning, while reducing dependence on dense human-written process supervision.

The main risk is that stronger long-horizon reasoning may also improve dual-use capabilities when integrated into external tools or autonomous systems. SAGE improves training-time reasoning stability, but it does not guarantee factual correctness, harmlessness, fairness, or robustness under distribution shift. Deployments in consequential settings should therefore include domain-specific validation, uncertainty estimation, human oversight, and task-level safety checks.

\section{Safeguards}
\label{sec:safeguards}

This work does not release a new pretrained foundation model, user-facing agent, scraped dataset, or deployment system. The released artifacts are limited to code, training and evaluation scripts, and reproducible reasoning resources. The experiments use public benchmarks and synthetic or symbolic reasoning tasks, without private user data or human-subject data collection.

SAGE is intended as a research framework, not a deployment-ready safety mechanism. Future extensions involving external tools, web access, code execution, or autonomous planning should add safeguards such as sandboxing, rate limits, misuse monitoring, restricted access for high-risk capabilities, and domain-specific safety evaluation before release.

\section{LLM usage}
\label{sec:llm}
Large language models were used only for language polishing and wording refinement. 
Specifically, LLM assistance was limited to improving grammar, clarity, conciseness, and presentation of author-written text. 
All scientific ideas, problem formulation, theoretical results, proofs, algorithms, experimental design, implementation, data processing, numerical results, analysis, and conclusions were produced and verified by the authors. 
LLMs were not used to generate experimental results, fabricate data, perform reviewer simulation for reported claims, or make autonomous scientific decisions. 
The authors take full responsibility for the final content of the paper.


\newpage
\section*{NeurIPS Paper Checklist}
\begin{enumerate}

\item {\bf Claims}
    \item[] Question: Do the main claims made in the abstract and introduction accurately reflect the paper's contributions and scope?
    \item[] Answer: \answerYes{} 
    \item[] Justification: The abstract and introduction state the main theoretical, methodological, and empirical claims, including SCA as a theoretical lens, SAGE as the proposed framework, and evaluation across 12 benchmarks and 10 model backbones. The claims are scoped to long-horizon reasoning under sparse-reward regimes and are supported by the theoretical analysis and experiments.

\item {\bf Limitations}
    \item[] Question: Does the paper discuss the limitations of the work performed by the authors?
    \item[] Answer: \answerYes{} 
    \item[] Justification: The paper includes a dedicated Limitations paragraph discussing dependence on training-time structural priors, the distinction between exact symbolic admissibility and approximate learned proxies in less structured domains, and additional training-time overhead in Appendix~\ref{sec:limit}.

\item {\bf Theory assumptions and proofs}
    \item[] Question: For each theoretical result, does the paper provide the full set of assumptions and a complete (and correct) proof?
    \item[] Answer: \answerYes{} 
    \item[] Justification: The paper states the assumptions for the theoretical results in the relevant propositions and theorem, and provides complete proofs in the appendix. The theoretical claims are explicitly scoped to symbolic settings and conditionally to settings where learned potentials separate productive and unproductive prefixes in Appendix~\ref{app:sca}-Appendix~\ref{app:sage_soft}.

    \item {\bf Experimental result reproducibility}
    \item[] Question: Does the paper fully disclose all the information needed to reproduce the main experimental results of the paper to the extent that it affects the main claims and/or conclusions of the paper (regardless of whether the code and data are provided or not)?
    \item[] Answer: \answerYes{} 
    \item[] Justification: The paper describes the evaluated models, benchmarks, baselines, metrics, rollout settings, ablations, and implementation details, and provides an anonymized code repository for reproducing the main results in Appendix~\ref{app:exp}.

\item {\bf Open access to data and code}
    \item[] Question: Does the paper provide open access to the data and code, with sufficient instructions to faithfully reproduce the main experimental results, as described in supplemental material?
    \item[] Answer: \answerYes{} 
    \item[] Justification: The paper provides an anonymized code repository. The experiments use public benchmarks and a reproducible construction protocol for the Andrews--Curtis presentations; scripts and instructions are provided with the released code.

\item {\bf Experimental setting/details}
    \item[] Question: Does the paper specify all the training and test details (e.g., data splits, hyperparameters, how they were chosen, type of optimizer) necessary to understand the results?
    \item[] Answer: \answerYes{} 
    \item[] Justification: The main text describes the evaluated model families, datasets, baselines, and metrics, while the appendix~\ref{app:exp} provides training-time details, filtering procedures, hyperparameter settings, and ablation protocols.

\item {\bf Experiment statistical significance}
    \item[] Question: Does the paper report error bars suitably and correctly defined or other appropriate information about the statistical significance of the experiments?
    \item[] Answer: \answerYes{} 
    \item[] Justification: The paper reports five-seed mean-standard deviation results in the appendix for the main benchmark tables in Appendix~\ref{app:addition}. The reported variability corresponds to independent training/evaluation runs under the same experimental conditions.

\item {\bf Experiments compute resources}
    \item[] Question: For each experiment, does the paper provide sufficient information on the computer resources (type of compute workers, memory, time of execution) needed to reproduce the experiments?
    \item[] Answer: \answerYes{} 
    \item[] Justification: The paper reports the compute resources required for the experiments in Appendix~\ref{sec:compute}.
    
\item {\bf Code of ethics}
    \item[] Question: Does the research conducted in the paper conform, in every respect, with the NeurIPS Code of Ethics \url{https://neurips.cc/public/EthicsGuidelines}?
    \item[] Answer: \answerYes{} 
    \item[] Justification: The research conforms to the NeurIPS Code of Ethics. The work uses public benchmarks and synthetic/symbolic reasoning tasks, does not involve human-subject data collection, and preserves anonymity in the released materials.

\item {\bf Broader impacts}
    \item[] Question: Does the paper discuss both potential positive societal impacts and negative societal impacts of the work performed?
    \item[] Answer: \answerYes{} 
    \item[] Justification: The paper discusses potential positive impacts, including improving reliable long-horizon reasoning and reducing dependence on dense process supervision, as well as potential risks from stronger reasoning models, such as misuse in automated generation or decision-support contexts in Appendix~\ref{sec:broader_impact}.
    
\item {\bf Safeguards}
    \item[] Question: Does the paper describe safeguards that have been put in place for responsible release of data or models that have a high risk for misuse (e.g., pre-trained language models, image generators, or scraped datasets)?
    \item[] Answer: \answerYes{} 
    \item[] Justification: The paper provides a safeguards section in Appendix~\ref{sec:safeguards}.

\item {\bf Licenses for existing assets}
    \item[] Question: Are the creators or original owners of assets (e.g., code, data, models), used in the paper, properly credited and are the license and terms of use explicitly mentioned and properly respected?
    \item[] Answer: \answerYes{}
    \item[] Justification: The paper cites the original sources for the datasets, model backbones, and baseline methods used in the experiments. The released code and documentation include the licenses or terms of use for existing assets where available.

\item {\bf New assets}
    \item[] Question: Are new assets introduced in the paper well documented and is the documentation provided alongside the assets?
    \item[] Answer: \answerYes{} 
    \item[] Justification: The paper introduces and releases code and experimental resources for SAGE, including scripts for constructing and evaluating the Andrews--Curtis reasoning instances. These assets are documented in the anonymized repository with instructions for reproducing the reported experiments.

\item {\bf Crowdsourcing and research with human subjects}
    \item[] Question: For crowdsourcing experiments and research with human subjects, does the paper include the full text of instructions given to participants and screenshots, if applicable, as well as details about compensation (if any)? 
    \item[] Answer: \answerNA{} 
    \item[] Justification: The paper does not involve crowdsourcing, human-subject experiments, or participant compensation.

\item {\bf Institutional review board (IRB) approvals or equivalent for research with human subjects}
    \item[] Question: Does the paper describe potential risks incurred by study participants, whether such risks were disclosed to the subjects, and whether Institutional Review Board (IRB) approvals (or an equivalent approval/review based on the requirements of your country or institution) were obtained?
    \item[] Answer: \answerNA{} 
    \item[] Justification: The paper does not involve human-subject research, so IRB approval or equivalent review is not applicable.

\item {\bf Declaration of LLM usage}
    \item[] Question: Does the paper describe the usage of LLMs if it is an important, original, or non-standard component of the core methods in this research? Note that if the LLM is used only for writing, editing, or formatting purposes and does \emph{not} impact the core methodology, scientific rigor, or originality of the research, declaration is not required.
    \item[] Answer: \answerYes{} 
    \item[] Justification: The paper discloses LLM usage in Appnedix~\ref{sec:llm}. LLMs were used only for wording refinement, including grammar, clarity, conciseness, and presentation. 

\end{enumerate}

\end{document}